\documentclass{article}

\usepackage{amssymb}
\usepackage{amsthm}
\usepackage{amsmath}
\usepackage{stmaryrd}
\usepackage{color}

\newtheorem{theorem}{Theorem}[section]
\newtheorem{lemma}[theorem]{Lemma}

\theoremstyle{definition}

\theoremstyle{remark}
\newtheorem{remark}[theorem]{Remark}

\numberwithin{equation}{section}

\usepackage{framed}

\usepackage{graphicx}

\def\RR{\mathbb{R}}
\def\CC{\mathbb{C}}
\def\sinc{\mbox{sinc}}

\newcommand{\abs}[1]{\left\lvert #1 \right\rvert}

\newcommand{\norm}[1]{\left\lVert#1\right\rVert}

\title{Lipschitz Properties for Deep Convolutional Networks}

\author{Radu Balan\footnote{Department of Mathematics and Center for Scientific Computation and Mathematical Modeling, University of Maryland, College Park, MD 20742 , rvbalan@math.umd.edu} , Maneesh Singh \footnote{Image and Video Analytics, Verisk Analytics, 545 Washington Boulevard, Jersey City, NJ 07310,
 Maneesh.Singh@verisk.com} , Dongmian Zou \footnote{Department of Mathematics and Center for Scientific Computation and Mathematical Modeling, University of Maryland, College Park, MD 20742, zou@math.umd.edu} }

\date{\today}

\begin{document}

\maketitle

\begin{abstract}
In this paper we discuss the stability properties of convolutional neural networks. Convolutional neural networks are widely used in machine learning.  In classification they are mainly used as feature extractors. Ideally, we expect similar features when the inputs are from the same class. That is, we hope to see a small change in the feature vector with respect to a deformation on the input signal. This can be established mathematically, and the key step is to derive the Lipschitz properties. Further, we establish that the stability results can be extended for more general networks. We give a formula for computing the Lipschitz bound, and compare it with other methods to show it is closer to the optimal value.
\end{abstract}





\section{Introduction}


Recently convolutional neural networks have enjoyed tremendous success in many applications in image and signal processing. According to \cite{Lecun15}, a general convolutional network contains three types of layers: convolution layers, detection layers, and pooling layers. In \cite{Mallat12}, Mallat proposes the scattering network, which is a tree-structured convolutional neural network whose filters in convolution layers are wavelets. Mallat proves that the scattering network satisfies two important properties: (approximately) invariance to translation and stabitity to deformation. However, for those properties to hold, the wavelets must satisfy an admissibility condition. This restricts the adaptability of the theory. The authors in \cite{WB15,WB16} use a slightly different setting to relax the conditions. They consider sets of filters that form semi-discrete frames of upper frame bound equal to one. They prove that deformation stability holds for signals that satisfy certain conditions.

In both settings, the deformation stability is a consequence of the Lipschitz property of the network, or feature extractor. The Lipschitz property in itself is important even if we do not consider deformation of the form described in \cite{Mallat12}. In \cite{Szegedy13}, the authors detect some instability of the AlexNet by generating images that are easily recognizable by nude eyes but cause the network to give incorrect classification results. They partially attribute the instability to the large Lipschitz bound of the AlexNet. It is thus desired to have a formula to compute the Lipschitz bound in case the upper frame bound is not one.

The lower bound in the frame condition is not used when we analyze the stability properties for scattering networks. In \cite{WB16} the authors conjectured that it has to do with the distinguishability of the two classes for classification. However, certain loss of information should be allowed for classification tasks. A lower frame bound is too strong in this case since it has most to do with injectivity. In this paper, we only consider the semi-discrete Bessel sequence, and discuss a convolutional network of finite depth.

Merging is widely used in convolutional networks. Note that practitioners use a concatenation layer (\cite{Szegedy15}) but that is just a concatenation of vectors and is of no mathematical interest. Nevertheless, aggregation by $p$-norms and multiplication is frequently used in networks and we 
still obtain stability to deformation in those cases and the Lipschitz bound increases only by a factor depending on the number of filters to be aggregated.

The organization of this paper is as follows. In Section 2, we introduce the scattering network and state a general Lipschitz property. In Section 3, we discuss the aggregation of filters using $p$-norms or pointwise multiplication. In Section 4, we use examples of networks to compare different methods for computing the Lipschitz constants.

\section{Scattering Network}



\begin{figure}[ht!]
 \centering
 \includegraphics[width=\textwidth]{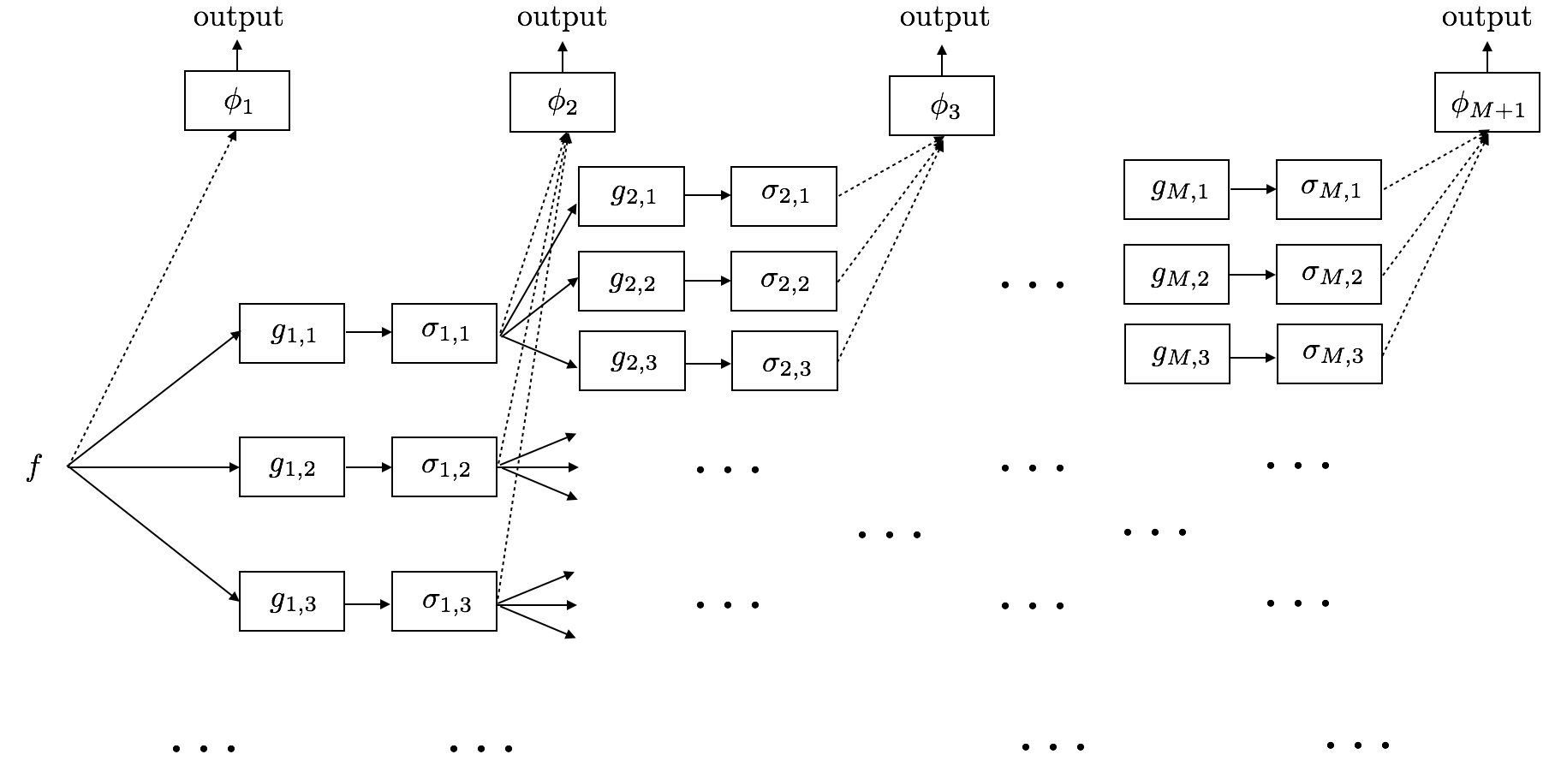}
 \caption{Structure of the scattering network of depth $M$}
 \label{fig:struct1}
\end{figure}


We first review the theory developed by the authors in \cite{Mallat12, WB15, WB16} and give a more general result. Figure \ref{fig:struct1} shows a typical scattering network. $f$ denotes an input signal (commonly in $L^2$ or $l^2$, for our discussion we take $f \in L^2(\mathbb{R}^d)$). $g_{m,l}$'s and $\phi_m$'s are filters and the corresponding blocks symbolizes the operation of doing convolution with the filter in the block. The blocks marked $\sigma_{m,l}$ illustrate the action of a nonlinear function. This structure clearly shows the three stages of a convolutional neural network: the $g_{m,l}$'s are the convolution stage; the $\sigma_{m,l}$'s are the detection stage; the $\phi_m$'s are the pooling stage.

The output of the network in Figure \ref{fig:struct1} is the collection of outputs of each layer. To represent the result clearly, we introduce some notations first. 

We call an ordered collection of filters $g_{m,l} \in L^1(\mathbb{R}^d)$ connected in the network starting from $m=1$ a path, say $q = ( g_{1,l_1}, g_{2,l_2}, \cdots, g_{M_q,l_{M_q}} )$ (for brevity we also denote it as $q = ((1,l_1),(2,l_2), \cdots, (M_q,l_{M_q}))$, and in this case we denote $\abs{q} = M_q$ to be the number of filters in the collection. We call $\abs{q}$ the length of the path. For $q = \emptyset$, we say that $\abs{q} = 0$. The largest possible $\abs{q}$, say $M$, is called the depth of the network. For each $m = 1,2,\cdots,M+1$, there is an output-generating atom $\phi_m \in L^1(\mathbb{R}^d)$, which is usually taken to be a low-pass filter. $\phi_1$ generates an output from the original signal $f$, and $\phi_m$ generates an output from a filter $g_{m-1,l_m}$ in the ($m-1$)'s layer, for $2 \leq m \leq M+1$. It is clear that a scattering network of finite depth is uniquely determined by $\phi_m$'s and the collection $Q$ of all paths. We use $G_m$ to denote the set of filters in the $m$-th layer. For a fixed $q$ with $\abs{q} = m$, we use $G_{m+1}^q$ to denote the set of filters in the ($m+1$)'s layer that are connected with $q$. Thus $G_{m+1}$ is a disjoint union of $G_{m+1}^q$'s:
\begin{equation*}
G_{m+1} = \mathop{\dot{\bigcup}} G_{m+1}^q ~.
\end{equation*}

$\sigma_{m,l}: \CC \rightarrow \CC$ are Lipschitz continuous functions with Lipschitz bound no greater than $1$. That is,
\begin{equation*}
\norm{\sigma_{m,l}(y)-\sigma_{m,l}(\tilde{y})}_2 \leq \norm{y-\tilde{y}}_2
\end{equation*}
for any $y$, $\tilde{y} \in L^2(\RR^d)$. The Lip-$1$ condition is not restrictive since any other Lipschitz constant can be absorbed by the proceeding $g_{m,l}$ filters. 

The scattering propagator $U[q]: L^2(\mathbb{R}^d) \rightarrow L^2(\mathbb{R}^d)$ for a path $q = ( g_{1,l_1}, g_{2,l_2},$ $\cdots, g_{M_q,l_{M_q}} )$ is defined to be
\begin{equation}
\label{eq:besseldef}
U[q]f := \sigma_{M_q,l_{M_q}} \Bigg( \sigma_{2,l_2} \Big( \sigma_{1,l_1}(f \ast g_{1,l_1}) \ast g_{2,l_2} \Big) \ast \cdots  \ast g_{M_q,l_{M_q}} \Bigg) ~.
\end{equation}
If $q = \emptyset$, then by convention we say $U[\emptyset]f := f$.

Given an input $f \in L^2(\mathbb{R}^d)$, the output of the network is the collections $\Phi(f) := \{U[q]f \ast \phi_{M_q+1} \}_{q \in Q}$. The norm $||| \cdot |||$ is defined by
\begin{equation}
\label{eq:framedef}
||| \Phi(f) ||| := \left( \sum_{q \in Q} \norm{ U[q]f \ast \phi_{M_q+1} }_2^2 \right) ^{\frac{1}{2}} ~.
\end{equation}

Given a collection of filters $\{ g_i \}_{i \in \mathcal{I}}$ where the index set $\mathcal{I}$ is at most countable and for each $i$, $g_i \in L^1(\mathbb{R}^d) \cap L^2(\mathbb{R}^d)$,
$\{g_i\}_{i \in \mathcal{I}}$ is said to form the atoms of a semi-discrete Bessel sequence if there exists a constant $B > 0$ for which
\begin{equation*}
\sum_{i \in \mathcal{I}} \norm{f \ast g_i}_2^2 \leq B \norm{f}_2^2
\end{equation*}
for any $f \in L^2$. In this case, $\{ g_i \}_{i \in \mathcal{I}}$ is said to form the atoms of a semi-discrete frame if in addition there exists a constant $A > 0$ for which 
\begin{equation*}
A \norm{f}_2^2 \leq \sum_{i \in \mathcal{I}} \norm{f \ast g_i}_2^2 \leq B \norm{f}_2^2
\end{equation*}
for any $f \in L^2$.

Conditions (\ref{eq:besseldef}) and (\ref{eq:defBessel}) can be achieved for a larger class of filters. Specifically, we shall introduce a Banach algebra in (\ref{def:banachalgebra}), where the Bessel bound is naturally defined.

Throughout this paper, we adapt the definition of Fourier transform of a function $f$ to be
\begin{equation}
\hat{f}(\omega) = \int_{\RR^d} f(x) e^{-2 \pi i \omega x} dx ~.
\end{equation}
The dilation of $f$ by a factor $\lambda$ is defined by
\begin{equation}
f_{\lambda} (x) = \lambda f(\lambda x) ~.
\end{equation}

The first result of this paper compiles and extends previous results obtained in \cite{Mallat12,WB15,WB16}.

\begin{theorem}[See also \cite{Mallat12,WB15,WB16}]
\label{prop:lip}
Suppose we have a scattering network of depth $M$. For each $m = 1, 2, \cdots, M+1$,
\begin{equation}\label{eq:defBessel}
B_m = \max_{q: \abs{q} = m-1} \norm{ \Big( \sum_{g_{m,l} \in G_m^q} \abs{\hat{g}_{m,l}}^2 \Big) + \abs{\hat{\phi}_m}^2 }_{\infty} < \infty
\end{equation}
with the understanding that
$B_{M+1} = \norm{\hat{\phi}_{M+1}}_{\infty}^2$ (that is, $G_{M+1}^q = \emptyset$).
Then the corresponding feature extractor $\Phi$
is Lipschitz continuous in the following manner:
\begin{equation*}
|||\Phi(f)-\Phi(h)||| \leq \left( \prod_{m=1}^{M+1} {\tilde{B}}_m \right) ^{\frac{1}{2}} \norm{f-h}_2 ~,
\end{equation*}
where
\begin{equation}\label{def:Bm}
\tilde{B}_1 = B_1, \quad \tilde{B}_m = \max\{1,B_m\} ~~\mbox{for}~ m \geq 2 ~.
\end{equation}
\end{theorem}

\begin{proof}
First we prove a lemma.
\begin{lemma}\label{lem:nonexpansive}
With the settings in Theorem \ref{prop:lip}, for $0 \leq m \leq M-1$, we have
\begin{equation}\label{eq:useBessel}
\begin{aligned}
& \sum_{\abs{q} = m+1} \norm{U[q]f-U[q]h}_2^2 + \sum_{\abs{q} = m} \norm{U[q]f \ast \phi_{m+1} - U[q]h \ast \phi_{m+1}}_2^2 \\
& ~\leq~ \sum_{\abs{q} = m} {B}_{m+1} \norm{U[q]f-U[q]h}_2^2 ~;
\end{aligned}
\end{equation}
for $m = M$, we have
\begin{equation}\label{eq:useBessel2}
\sum_{\abs{q} = M} \norm{U[q]f \ast \phi_{m+1} - U[q]h \ast \phi_{m+1}}_2^2 \leq \sum_{\abs{q} = M} {B}_{M+1} \norm{U[q]f-U[q]h}_2^2 ~.
\end{equation}
\end{lemma}
\begin{proof}[Proof of Lemma \ref{lem:nonexpansive}]
Let $q$ be a path with $\abs{q} = m < M$. We go one layer deeper to get
\begin{equation}\label{eq:useBessel1.1}
\begin{aligned}
& \sum_{q' \in q \times G_{m+1}^q} \norm{U[q']f-U[q']h}_2^2 + \norm{U[q]f \ast \phi_{m+1} - U[q]h \ast \phi_{m+1}}_2^2 \\
~=~ & \sum_{g_{m+1,l} \in G_{m+1}^q} \norm{\sigma_{m+1,l} ( U[q]f \ast g_{m+1,l} ) - \sigma_{m+1,l} ( U[q]h \ast g_{m+1,l} ) }_2^2 + \\
& \qquad \norm{U[q]f \ast \phi_{m+1} - U[q]h \ast \phi_{m+1}}_2^2 \\
~\leq~ & \sum_{g_{m+1,l} \in G_{m+1}^q} \norm{U[q]f \ast g_{m+1,l}-U[q]h \ast g_{m+1,l}}_2^2 + \norm{U[q]f \ast \phi_{m+1} - U[q]h \ast \phi_{m+1}}_2^2 \\
~=~ & \sum_{g_{m+1,l} \in G_{m+1}^q} \norm{(U[q]f-U[q]h) \ast g_{m+1,l}}_2^2 + \norm{(U[q]f-U[q]h) \ast \phi_{m+1}}_2^2 ~.
\end{aligned}
\end{equation}
Sum over all $q$ with length $m$, we have
\begin{equation}\label{eq:useBessel1.2}
\begin{aligned}
& \sum_{\abs{q} = m+1} \norm{U[q]f-U[q]h}_2^2 + \sum_{\abs{q} = m} \norm{U[q]f \ast \phi_{m+1} - U[q]h \ast \phi_{m+1}}_2^2 \\
~\leq~ & \sum_{\abs{q} = m} \sum_{g_{m+1,l} \in G_{m+1}} \norm{(U[q]f-U[q]h) \ast g_{m+1,l}}_2^2 + \\
& \qquad \sum_{\abs{q} = m} \norm{U[q]f \ast \phi_{m+1} - U[q]h \ast \phi_{m+1}}_2^2 \\
~\leq~ & \sum_{\abs{q} = m} B_{m+1} \norm{U[q]f-U[q]h}_2^2 ~,
\end{aligned}
\end{equation}
which follows the Bessel inequality by the definition of the $B_m$'s. For m = M, directly following the Young's inequality we have (\ref{eq:useBessel2}).
\end{proof}
We now continue with the proof of Theorem \ref{prop:lip}. The inqualities (\ref{eq:useBessel}) and (\ref{eq:useBessel2}) have two consequences. First, summing over $m = 0, \cdots, M$ we have
\begin{equation}\label{eq:useBessel3}
\begin{aligned}
& \sum_{m = 0}^M \sum_{\abs{q} = m} \norm{U[q]f \ast \phi_{m+1} - U[q]h \ast \phi_{m+1}}_2^2 \\
\leq ~ {B}_1 & \norm{f-h}_2^2 + \sum_{m = 1}^M ({B}_{m+1}-1) \sum_{\abs{q} = m} \norm{U[q]f-U[q]h}_2^2 ~;
\end{aligned}
\end{equation}
second, we have for each $m = 0, \cdots, M-1$ that
\begin{equation}\label{eq:useBessel4}
\sum_{\abs{q} = m+1} \norm{U[q]f-U[q]h}_2^2 \leq \sum_{\abs{q} = m} {B}_{m+1} \norm{U[q]f-U[q]h}_2^2 ~.
\end{equation}
Therefore, put (\ref{eq:useBessel3}) and (\ref{eq:useBessel4}) together, noting that $B_m \leq \tilde{B}_m$ for each m, we have
\begin{equation}\label{eq:useBessel5}
\begin{aligned}
& \sum_{m = 0}^M \sum_{\abs{q} = m} \norm{U[q]f \ast \phi_{m+1} - U[q]h \ast \phi_{m+1}}_2^2 \\
~\leq~ & \tilde{B}_1 \norm{f-h}_2^2 + \sum_{m = 1}^M (\tilde{B}_{m+1}-1) \left( \prod_{m' = 1}^m \tilde{B}_{m'} \right) \norm{f-h}_2^2 \\
~\leq~ & \left( \prod_{m=1}^{M+1} \tilde{B}_m \right) \norm{f-h}_2^2 ~.
\end{aligned}
\end{equation}
We complete the proof by observing that the uppermost object in Inequality (\ref{eq:useBessel5}) is nothing but $|||\Phi(f) - \Phi(h)|||^2$.

\end{proof}

\begin{remark}
\cite{WB15,WB16} consider the case where each filter in the $m$-th layer in connected to all the frame vectors from the pre-designed frame for the ($m+1$)-th layer. In practical uses then, a dimension reduction process needs to be done to select a few branches from the numerous tributaries due to such a design manner (see \cite{BM13}). Also, the authors of \cite{WB15,WB16} assume that all the $B_m$'s are less than or equal to one. As can be seen in the above proof, this assumption is not needed. 
\end{remark}

\begin{remark}
The infinite-depth case is an immediate extension of the finite-depth case if $\prod \tilde{B}_m < \infty$.
\end{remark}

Theorem \ref{prop:lip}, together with Schur's test (for integral operators), lead to the following theorem, which implies the deformation stability of the corresponding network. The proof can be found in \cite{WB15}. We state this result for the completeness of this article.

\begin{theorem}[\cite{WB15}]\label{thm:main}
With the settings in Theorem \ref{prop:lip},
Let $H_R$ be the space of R-band-limited functions defined by
\begin{equation*}
H_R := \{ f \in L^2(\mathbb{R}^d): \mbox{supp}(\hat{f}) \subset B_R(0) \} ~.
\end{equation*}
Then for all $f \in H_R$, $\omega \in C(\mathbb{R}^d,R)$, $\tau \in C^1(\mathbb{R}^d,R)$ with $\norm{D \tau}_{\infty} \leq (2d)^{-1}$,
\begin{equation*} 
|||\Phi(f) - \Phi(F_{\tau, \omega}f)||| \leq C(R\norm{\tau}_{\infty}+\norm{\omega}_{\infty})\norm{f}_2 ~,
\end{equation*}
where $F_{\tau, \omega}f)$ is the deformed version of $f$ defined by
\begin{equation}\label{def:deformation}
F_{\tau, \omega}f(x) := e^{2\pi i \omega(x)} f(x - \tau(x)) ~.
\end{equation}
\end{theorem}

Note that the Lipschitz property of $\sigma_{m,l}$'s is not necessary in some cases. For instance, if we use $\abs{\cdot}^2$ in place of all the $\sigma_{m,l}$'s, as illustrated in Figure \ref{fig:struct1_2}. Then the training process would deal with smooth functions that are not Lipschitz. To guarantee a finite Lipschitz constant for $\Phi$ we need to control the $L^{\infty}$ norm of the input. 

\begin{figure}[ht!]
 \centering
 \includegraphics[width=\textwidth]{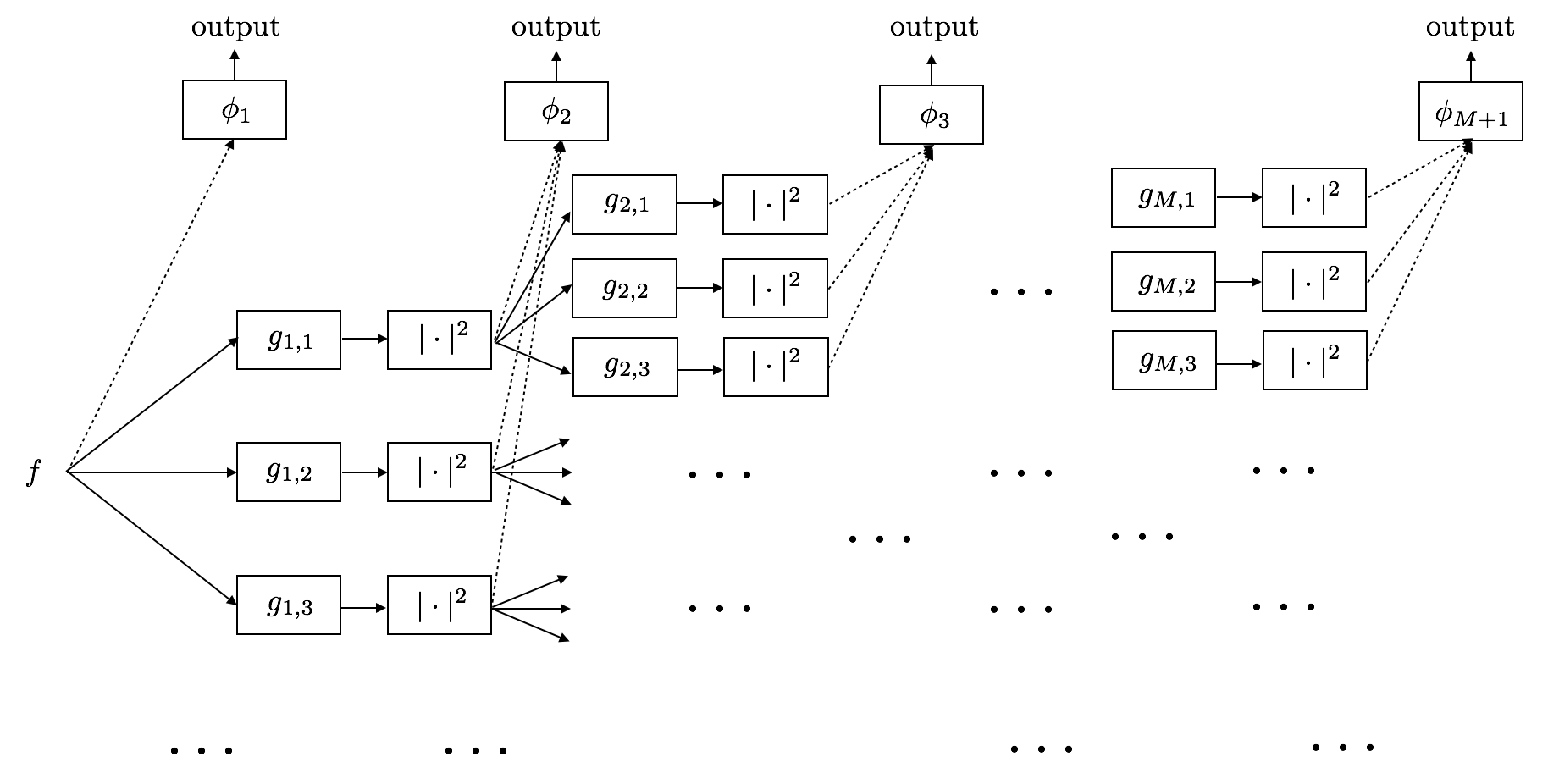}
 \caption{Structure of the scattering network of depth $M$ with nonlinearity $\abs{\cdot}^2$}
 \label{fig:struct1_2}
\end{figure}

\begin{theorem}\label{prop:lipSquare}
Consider the settings in Theorem \ref{prop:lip}, where $\sigma_{m,l}$'s are replaced with $\abs{\cdot}^2$ (see Figure \ref{fig:struct1_2}).
Suppose there is a constant $R>0$ for which $\norm{g_{m,l}}_1 \leq \frac{ \min \{ 1,2\sqrt{R} \} }{2R}$ for all $m$, $l$.
Then the corresponding feature extractor $\Phi$ 
is Lipschitz continuous on the ball of radius $R$ under infinity norm in the following manner:
\begin{equation*}
|||\Phi(f)-\Phi(h)||| \leq \left( \prod_{m=1}^{M+1} \tilde{B}_m \right) ^{\frac{1}{2}} \norm{f-h}_2 ~.
\end{equation*}
for any $f$, $h \in L^2(\mathbb{R}^d)$ with $\norm{f}_{\infty} \leq R$, $\norm{h}_{\infty} \leq R$, where $\tilde{B}_m$'s are defined as in (\ref{def:Bm}) and (\ref{eq:defBessel}). 
\end{theorem}

\begin{remark}
In the case of deformation, $h$ is given by $h = F_{\tau,\omega}$ as defined in (\ref{def:deformation}). If $f$ satisfies the $L^{\infty}$ condition $\norm{f}_{\infty} = R$, so does $h$, since $\norm{h}_{\infty} = \norm{f}_{\infty}$.
\end{remark}

\begin{proof}
Notice that $\frac{ \min \{ 1,2\sqrt{R} \} }{2R} = \min \{ \frac{1}{\sqrt{R}} , \frac{1}{2R} \}$. Hence $\norm{g_{m,l}}_1 \leq 1/\sqrt{R}$ and $\norm{g_{m,l}}_1 \leq 1/2R$. We observe that for any path $q$ with length $\abs{q} = m \geq 1$, say $q = \left( (1,l_1),(2,l_2),\cdots,(M_q,l_{M_q}) \right)$, and for convenience denote $q_1 = \left( (1,l_1) \right)$, $q_2 = \left( (1,l_1),(2,l_2) \right)$, $\cdots$, $q_{M_q-1} = \left( (1,l_1),(2,l_2),\cdots,(M_q-1,l_{M_q-1}) \right)$,
we have
\begin{equation*}
\begin{aligned}
\norm{U[q]f}_{\infty} ~=~ & \norm{\abs{U[q_{M_q-1}]f \ast g_{M_q,l_{M_q}}}^2}_{\infty} \\
~\leq~ & \norm{U[q_{M_q-1}]f}_{\infty}^2 \norm{g_{M_q,l_{M_q}}}_1^2 \\
~\leq~ & \norm{U[q_{M_q-2}]f}_{\infty}^4 \norm{g_{M_q-1,l_{M_q-1}}}_1^4  \norm{g_{M_q,l_{M_q}}}_1^2 \\
~\leq~ & \cdots \\
~\leq~ & \norm{U[q_1]f}_{\infty}^{2^{M_q-1}} \prod_{j=2}^{M_q} \norm{g_{j,l_j}}_1^{2^{M_q-j+1}} \\
~\leq~ & \norm{f}_{\infty}^{2^{M_q}} \prod_{j=1}^{M_q} \norm{g_{j,l_j}}_1^{2^{M_q-j+1}}\\
~\leq~ & R^{2^{M_q}} \prod_{j=1}^{M_q} \left( \frac{1}{\sqrt{R}} \right) ^{2^{M_q-j+1}} \\
~=~ & R^{2^{M_q}} \cdot \left( \frac{1}{\sqrt{R}} \right) ^{(2^{M_q}-1)} \\
~=~ & R ~.
\end{aligned}
\end{equation*}
With this, let $q$ be a path of length $\abs{q} = m < M$, we have for each $l$ that
\begin{equation*}
\begin{aligned}
& \norm{\abs{U[q]f \ast g_{m+1,l}}^2 - \abs{U[q]h \ast g_{m+1,l}}^2}_2^2 \\
~=~ & \norm{ \left( \abs{U[q]f \ast g_{m+1,l}} + \abs{U[q]h \ast g_{m+1,l}} \right) \left( \abs{U[q]f \ast g_{m+1,l}} - \abs{U[q]h \ast g_{m+1,l}} \right) }_2^2 \\
~\leq~ & \norm{ \abs{U[q]f \ast g_{m+1,l}} + \abs{U[q]h \ast g_{m+1,l}} }_1^2 \norm{ \abs{U[q]f \ast g_{m+1,l}} - \abs{U[q]h \ast g_{m+1,l}} }_2^2 \\
~\leq~ & \left( \norm{U[q]f}_{\infty} + \norm{U[q]h}_{\infty} \right)^2 \norm{g_{m+1,l}}_1^2 \norm{ \abs{U[q]f \ast g_{m+1,l}} - \abs{U[q]h \ast g_{m+1,l}} }_2^2 \\
~\leq~ & (R+R)^2(1/2R)^2 \norm{ \abs{U[q]f \ast g_{m+1,l}} - \abs{U[q]h \ast g_{m+1,l}} }_2^2 \\
~=~ & \norm{ \abs{U[q]f \ast g_{m+1,l}} - \abs{U[q]h \ast g_{m+1,l}} }_2^2 \\
~\leq~ & \norm{U[q]f \ast g_{m+1,l} - U[q]h \ast g_{m+1,l} }_2^2 ~.
\end{aligned}
\end{equation*}
Therefore, 
\begin{equation*}
\begin{aligned}
& \sum_{q' \in q \times G_{m+1}^q} \norm{U[q']f-U[q']h}_2^2 + \norm{U[q]f \ast \phi_{m+1} - U[q]h \ast \phi_{m+1}}_2^2 \\
~\leq~ & \sum_{g_{m+1,l} \in G_{m+1}^q} \norm{U[q]f \ast g_{m+1,l}-U[q]h \ast g_{m+1,l}}_2^2 + \norm{U[q]f \ast \phi_{m+1} - U[q]h \ast \phi_{m+1}}_2^2 \\
~=~ & \sum_{g_{m+1,l} \in G_{m+1}^q} \norm{(U[q]f-U[q]h) \ast g_{m+1,l}}_2^2 + \norm{(U[q]f-U[q]h) \ast \phi_{m+1}}_2^2 ~.
\end{aligned}
\end{equation*}
Then by exactly the same inequality as (\ref{eq:useBessel1.2}),
for $0 \leq m \leq M-1$, 
\begin{equation*}
\begin{aligned}
& \sum_{\abs{q} = m+1} \norm{U[q]f-U[q]h}_2^2 + \sum_{\abs{q} = m} \norm{U[q]f \ast \phi_{m+1} - U[q]h \ast \phi_{m+1}}_2^2 \\
\leq & \sum_{\abs{q} = m} {B}_{m+1} \norm{U[q]f-U[q]h}_2^2 ~;
\end{aligned}
\end{equation*}
and for $m = M$,
\begin{equation*}
\sum_{\abs{q} = M} \norm{U[q]f \ast \phi_{m+1} - U[q]h \ast \phi_{m+1}}_2^2 \leq \sum_{\abs{q} = M} {B}_{M+1} \norm{U[q]f-U[q]h}_2^2 ~.
\end{equation*}

The rest of the proof is a minimal modification to that of Theorem \ref{prop:lip}. It is obvious that $\norm{f}_{\infty} = \norm{F_{\tau,\omega}}_{\infty}$.
\end{proof}


In most applications, the $L^{\infty}$-norm of the input is well bounded. For instance, normalized grayscale images have pixel valued between 0 and 1. Even if it is not the case, we can pre-filter the input by widely used sigmoid functions, such as $\tanh$. For instance, in the above case of $\abs{\cdot}^2$, we can use the structure as follows.

\begin{figure}[ht!]
 \centering
 \includegraphics[width=0.33\textwidth]{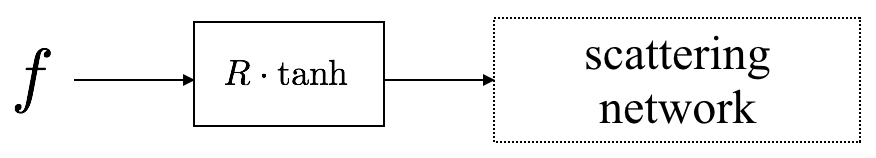}
 \caption{Restrict $\norm{f}_{\infty}$ at the first layer using $R \cdot \tanh$}
 \label{fig:struct4}
\end{figure}

\section{Filter Aggregation}


\subsection{Aggregation by taking norm across filters}

We use filter aggregation to model the pooling stage after convolution. In deep learning there are two widely used pooling operation, max pooling and average pooling. Max pooling is the operation of extracting local maximum of the signal, and can be modeled by an $L^{\infty}$-norm aggregation of copies of shifted and dilated signals. Average pooling is the operation of taking local average of the signal, and can be modeled by a $L^1$-norm aggregation of copies of shifted and dilated signals. When those pooling operations exist, it is still desired that the feature extractor is stable. We analyze this type of aggregation in detail as follows.

We consider filter aggregation by taking pointwise $p$-norms of the inputs. That is, suppose the inputs of the aggregation are $y_1, y_2, \cdots, y_L$ from $L$ different filters, the output is given by $( \sum_{l = 1}^L \abs{y_l}^p )^{1/p}$ for some $p$ with $1 \leq p \leq \infty$. Note that $y_1, y_2, \cdots, y_L$ are all $L^2$ functions and thus the output is also a $L^2$ function. A typical structure is illustrated in Figure \ref{fig:struct2}. Recall all the nonlinearities $\sigma_{m,l}$'s are assumed to be pointwise Lipschitz functions, with Lipschitz bound less than or equal to one.


\begin{figure}[ht!]
 \centering
 \includegraphics[width=\textwidth]{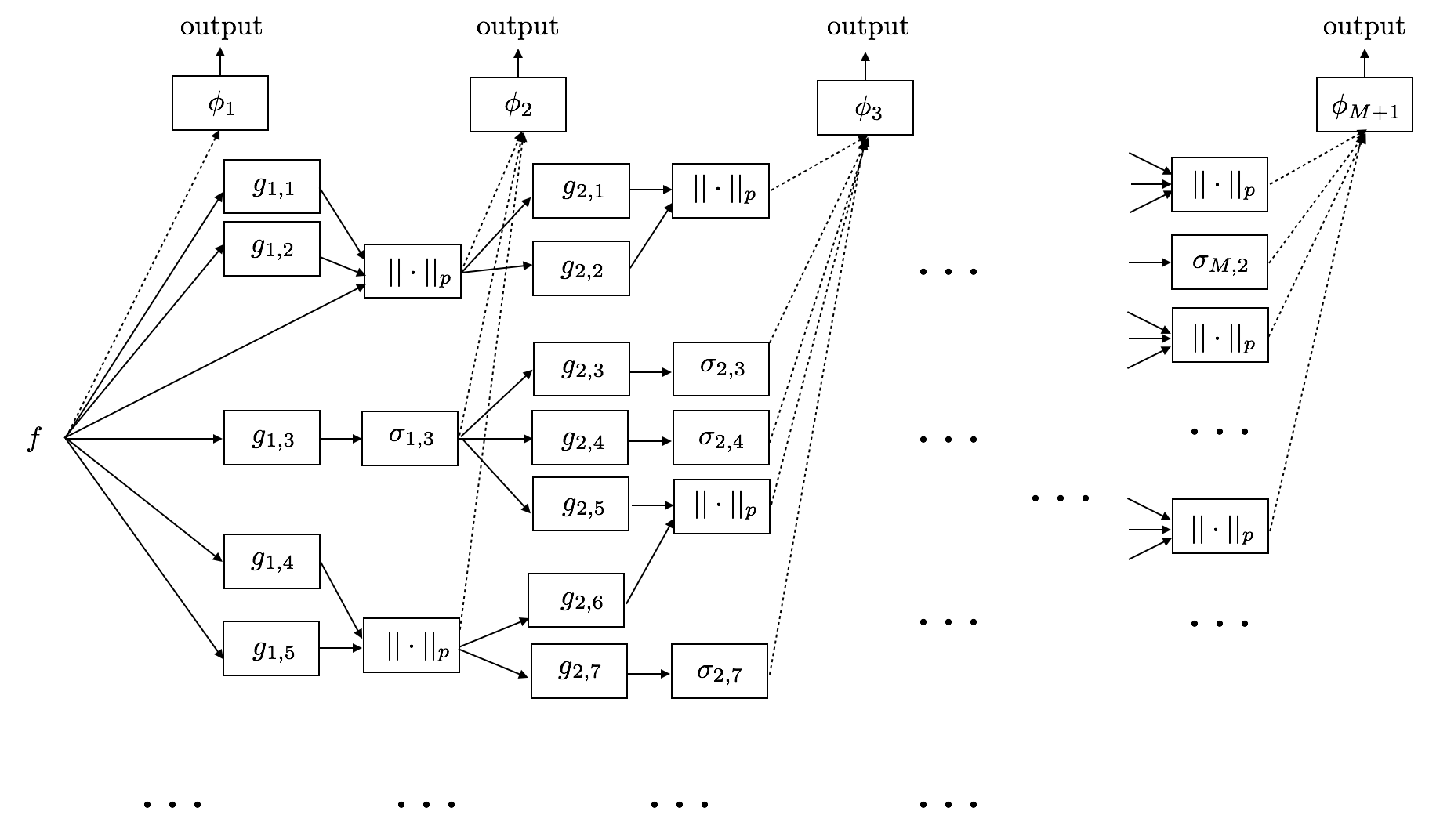}
 \caption{A typical structure of the scattering network with pointwise $p$-norms}
 \label{fig:struct2}
\end{figure}

Note that we do not necessarily aggregate filters in the same layer. For instance, in Figure \ref{fig:struct2}, $f \ast g_{1,1}$, $f \ast g_{1,2}$ are aggregated with $f$. Nevertheless, for the purpose of analysis it suffices to consider the case where the filters to be aggregated are in the same layer of the network. To see this, note that the equivalence relation in Figure \ref{fig:equiv}. We can coin a block which does not change the input (think of a $\delta$-function if we want to make the block ``convolutional''). 
Since a $\delta$-function is not in $L^1(\mathbb{R}^d)$, if we want to apply the theory we have to consider a larger space where the filters stay. In this case, it is natural to consider the Banach algebra
\begin{equation}
\label{def:banachalgebra}
\mathcal{B} = \left\{ f \in \mathcal{S}'(\RR^d), \norm{\hat{f}}_{\infty} < \infty \right\} ~.
\end{equation}

Without loss of generality, we can consider only networks in which the aggregation only takes inputs from the same layer.


\begin{figure}[ht!]
 \centering
 \includegraphics[width=0.5\textwidth]{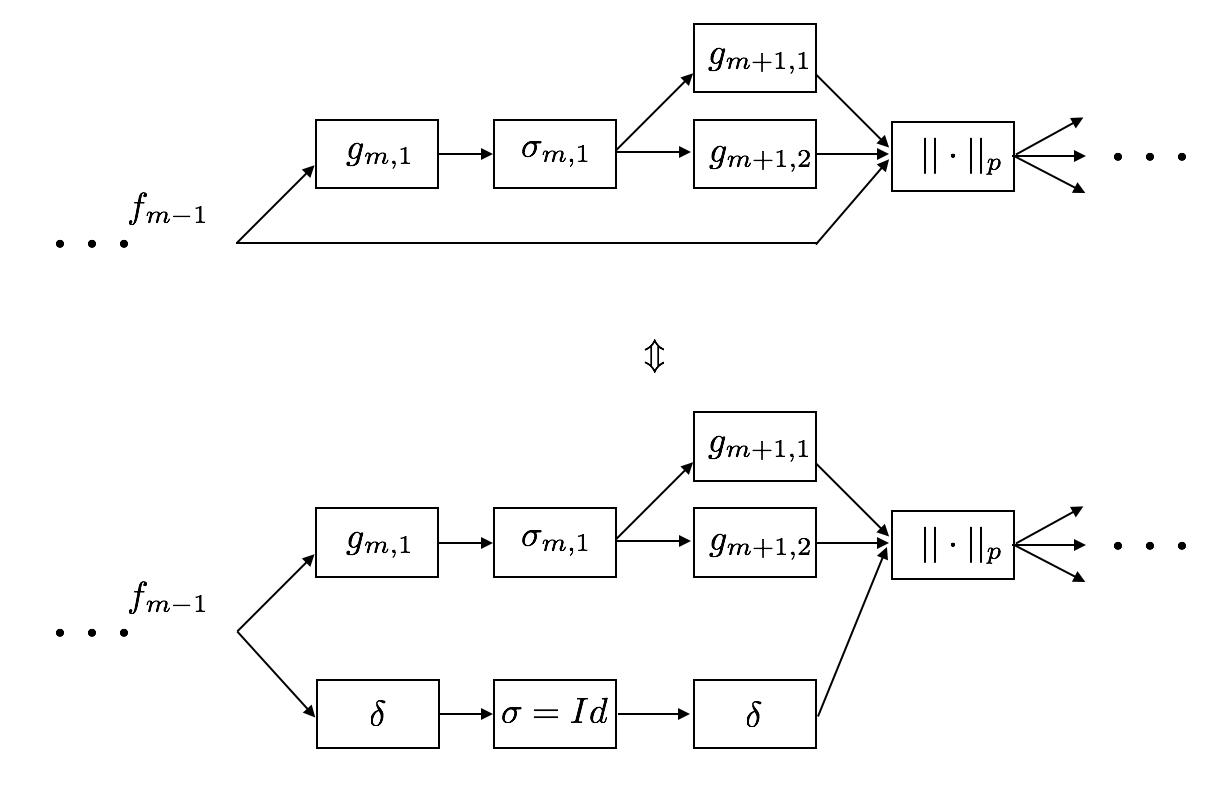}
 \caption{Equivalence for aggregating from different layers}
 \label{fig:equiv}
\end{figure}

Our purpose is to derive inequalities similar to (\ref{eq:useBessel}) and (\ref{eq:useBessel2}). We define a path $q$ to be a sequence of filters in the same manner as in Section 2. Note that by aggregating the filters we no longer have a scattering structure but a general convolutional network. That is, we might have two different filters in the $m$-th layer that flows into the same filter in the $(m+1)$-th layer. Although a scattering network with aggregation by the $p$-norm is still uniquely determined by the collection $Q$ of its paths, the notation $U[q]$ is meaningless since it does not take into account the aggregation. The output in this case may not depend on a single path.

Note that for each $m = 1, \cdots, M$, the $m$-th layer of filters is followed by blocks of $\norm{\cdot}_p$'s and nonlinearity $\sigma_{m,l}$'s. Let $\mu_m$ be the total number of the blocks in the $m$-th layer. Also take $\mu_0 = 1$. Further, we denote the blocks to be $K_{m,1}, \cdots, K_{m,\mu_m}$. For a block $K$ and a filter $g$, we denote $g \leftrightarrow K$ if they are connected in the network. For a block $K_{m,\lambda}$, $1 \leq m \leq M$, $1 \leq \lambda \leq \mu_m$, we denote $G_{m,\lambda}^{\mbox{\small{in}}}$ to be the collection of filters in the $m$-th layer that are connected to $K_{m,\lambda}$ (``in'' implies the filters ``flow into'' the block), and denote $G_{m+1,\lambda}^{\mbox{\small{out}}}$ to be the collection of filters in the ($m+1$)-th layer connected to $K_{m,\lambda}$. Then for each $m = 1, \cdots, M$, $G_m = \mathop{\dot{\bigcup}}_{1 \leq \lambda \leq \mu_m} G_{m,\lambda}^{\mbox{\small{in}}}$; also, for each $m = 1, \cdots, M-1$, $G_m = \mathop{\dot{\bigcup}}_{1 \leq \lambda' \leq \mu_{m+1}} G_{m,\lambda'}^{\mbox{\small{out}}}$.

We define the scattering propagator $\{ U_m^1, \cdots, U_m^{\mu_m} \}_{m=0}^M$ recursively as follows. Define $U_0^1 f := f$. Suppose $\{ U_m^1, \cdots, U_m^{\mu_m} \}$ has been defined for some $m < M$, then for each $\lambda = 1, \cdots, \mu_{m+1}$, we define
\begin{equation}\label{def:scatProp}
U_{m+1}^{\lambda} f := \left( \sum_{g_{m+1,l_{m+1}} \leftrightarrow K_{m+1,\lambda}} \abs{U_m^{\lambda'}f \ast g_{m+1,l_{m+1}}}^p \right)^{\frac{1}{p}} ~,
\end{equation}
where $\lambda'$ satisfies $g_{m+1,l_{m+1}} \leftrightarrow K_{m,\lambda'}$, which is unique by the structure of the network. Now the output $\Phi(f) := \{ U_m^{\lambda} \ast \phi_{m+1} \}_{0 \leq m \leq M, 1 \leq \lambda \leq m}$ is naturally defined. 

To proceed we first prove the following lemma.

\begin{lemma}\label{lem:aggNorm}
Let $\{g_{m,l}\}_{l=1}^L$ be the filters to be aggregated using $p$-norm with $1 \leq p \leq \infty$, then we have the following: suppose $\{f_{m-1,l}\}_{l=1}^L$ and $\{h_{m-1,l}\}_{l=1}^L$ are two sets of inputs to those filters and $f_m$ and $h_m$ are the outputs respectively, then
\begin{equation}
\norm{f_m-h_m}_2^2 \leq \max(1,L^{2/p-1}) \sum_{l=1}^L \norm{ \left( f_{m-1,l} - h_{m-1,l} \right) \ast g_{m,l} }_2^2 ~.
\end{equation}
\end{lemma}

\begin{proof}
For $1 \leq p \leq \infty$, applying $\abs{\norm{v_1}_p - \norm{v_2}_p} \leq \norm{v_1 - v_2}_p$ and $\norm{v_1}_p \leq \max(1,L^{1/p-1/2}) \norm{v_1}_2$ for any vectors $v_1$, $v_2$ of length $L$, we have
\begin{equation*}
\begin{aligned}
\norm{f_m-h_m}_2^2  ~=~ & \norm{ \left( \sum_{l=1}^L \abs{f_{m-1,l} \ast g_{m,l}}^p \right) ^{1/p} - \left( \sum_{l=1}^L \abs{h_{m-1,l} \ast g_{m,l}}^p \right) ^{1/p}}_2^2 \\
~\leq~ & \norm{ \left( \sum_{l=1}^L \abs{ \left( f_{m-1,l} - h_{m-1,l} \right) \ast g_{m,l}}^p \right) ^{1/p}}_2^2 \\
~\leq~ & \norm{ \max(1,L^{1/p-1/2}) \left( \sum_{l=1}^L \abs{\left( f_{m-1,l} - h_{m-1,l} \right) \ast g_{m,l}}^2 \right)^{1/2} }_2^2 \\
~=~ & \max(1,L^{2/p-1}) \int \sum_{l=1}^L \abs{\left( f_{m-1,l} - h_{m-1,l} \right) \ast g_{m,l}}^2 \\
~=~ & \max(1,L^{2/p-1}) \sum_{l=1}^L \norm{ \left( f_{m-1,l} - h_{m-1,l} \right) \ast g_{m,l} }_2^2 ~.
\end{aligned}
\end{equation*}
\end{proof}

With Lemma \ref{lem:aggNorm} we can compute for any $m = 0, \cdots, M$ that
\begin{equation*}
\begin{aligned}
& \sum_{\lambda=1}^{\mu_{m+1}} \norm{U_{m+1}^{\lambda}f - U_{m+1}^{\lambda}h}_2^2 \\
~\leq~ & \sum_{\lambda'=1}^{\mu_m} \sum_{l: g_{m+1,l} \in G_{m,\lambda'}^{\mbox{\tiny{out}}}} \max \left( 1,\abs{G_{m+1,\lambda}^{\mbox{\small{in}}}} ^{2/p-1} \right) \norm{U_m^{\lambda'}f \ast g_{m+1,l} - U_m^{\lambda'}h \ast g_{m+1,l}}_2^2 ~,
\end{aligned}
\end{equation*}
where for each $m$, $l$, $G_{m+1,\lambda}^{\mbox{\small{in}}}$ is the unique class of filters that contains $g_{m+1,l}$.
We can then proceed similar to Inequality (\ref{eq:useBessel1.1}) with minor changes. We get the following result on the Lipschitz properties for $\Phi$.

\begin{theorem}\label{prop:lipAggNorm}
Suppose we have a scattering network of depth $M$ including only $p$-norm aggregations. For $m = 1,2,\cdots,M+1$, set 
\begin{equation*}
B_m = \max_{1 \leq \lambda' \leq \mu_{m-1}} \norm{ \sum_{l:g_{m,l} \in G_{m,\lambda'}^{\mbox{\tiny{out}}}} \max \left( 1,\abs{G_{m+1,\lambda}^{\mbox{\small{in}}}} ^{2/p-1} \right) \abs{\hat{g}_{m,l}}^2 + \abs{\hat{\phi}_m}^2 }_{\infty} < \infty 
\end{equation*}
(with the understanding that 
$B_{M+1} = \norm{\hat{\phi}_{M+1}}_{\infty}^2$, that is, $G_{M+1,\lambda'}^{\mbox{\small{out}}} = \emptyset$ for any $1 \leq \lambda' \leq \mu_{M}$), 
where for each $m$, $l$, $G_{m,\lambda}^{\mbox{\small{in}}}$ is the unique class of filters that contains $g_{m,l}$ and $G_{m,\lambda}^{\mbox{\small{in}}}$ denotes its cardinal. 
Then the corresponding feature extractor $\Phi$ 
is Lipschitz continuous in the following manner:
\begin{equation*}
|||\Phi(f)-\Phi(h)||| \leq \left( \prod_{m=1}^{M+1} \tilde{B}_m \right) ^{\frac{1}{2}} \norm{f-h}_2 ~, ~~~\forall f, h \in L^2(\RR^d) ~,
\end{equation*}
where $\tilde{B}_m$'s are defined as in (\ref{def:Bm}) and (\ref{eq:defBessel}).
\end{theorem}


\subsection{Aggregation by pointwise multiplication}
In convolutional networks that includes time sequences, it is often useful to take the pointwise product of two intermediate outputs. For instance, in the Long Short-Term Memory (LSTM) networks introduced in \cite{HS97, SVSS15}, 
multiplication is used when we have two branches and want to use one branch for information extraction and the other for controlling, or so called ``gating''. A typical structure is illustrated in Figure \ref{fig:struct3}. The multiplication brings two outputs into one.


\begin{figure}[ht!]
 \centering
 \includegraphics[width=0.75\textwidth]{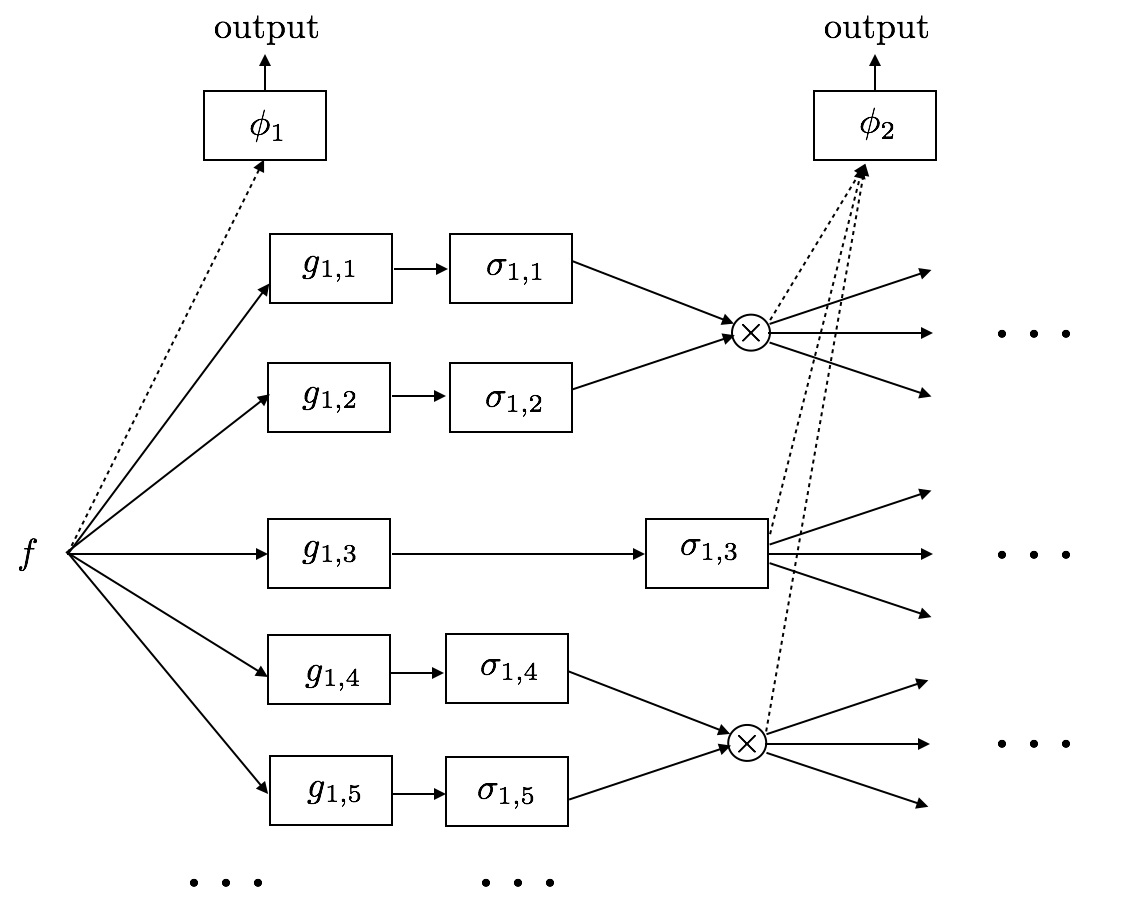}
 \caption{A typical structure of a scattering network with multiplication}
 \label{fig:struct3}
\end{figure}



Similar to the previous section, we consider multiplication blocks (if a filter is not followed by a multiplication block, such as $g_{1,3}$ in Figure \ref{fig:struct3}, we still consider a block after $\abs{\cdot}$), $J_{m,\lambda}$, $1 \leq m \leq M$, $1 \leq \lambda \leq \mu_m$. We define $G_{m,\lambda}^{\mbox{\small{in}}}$ and $G_{m+1,\lambda}^{\mbox{\small{out}}}$ to be the filters in the $m$-th and the ($m+1$)-th layer that are connected to $J_{m,\lambda}$, respectively. Note that $\abs{G_{m,\lambda}^{\mbox{\small{in}}}} \in \{1,2\}$. The scattering propagator $U_m^{\lambda}$'s and output generating operator $\Phi$ are defined similarly. The Lipschitz property is given by the following Theorem.

\begin{theorem}\label{prop:lipMult}
Suppose we have a scattering network of depth $M$ involving only pointwise multiplication blocks. For $m = 1,2,\cdots,M+1$,
\begin{equation*}
B_m = \max_{1 \leq \lambda \leq \mu_m} \norm{ \sum_{g_{m,l} \in G_{m,\lambda}^{\mbox{\tiny{out}}}} \abs{G_{m,\lambda}^{\mbox{\small{in}}}} \abs{\hat{g}_{m,l}}^2 + \abs{\hat{\phi}_m}^2 }_{\infty} < \infty
\end{equation*}
(with the understanding that 
$B_{M+1} = \norm{\hat{\phi}_{M+1}}_{\infty}^2$, that is, $G_{M+1,\lambda}^{\mbox{\small{out}}} = \emptyset$ for all $1 \leq \lambda \leq M$),
where for each $m$, $l$, $G_{m,\lambda}^{\mbox{\small{in}}}$ is the unique class of filters that contains $g_{m,l}$.  
Suppose $\norm{g_{m,l}}_1 \leq 1$ for all $m$, $l$.
Then the corresponding feature extractor $\Phi$ 
is Lipschitz continuous on the ball of radius $1$ under infinity norm in the following manner:
\begin{equation*}
|||\Phi(f)-\Phi(h)||| \leq \left( \prod_{m=1}^{M+1} \tilde{B}_m \right) ^{\frac{1}{2}} \norm{f-h}_2 ~,
\end{equation*}
for any $f$, $h \in L^2({\mathbb{R}^d})$ with $\norm{f}_{\infty} \leq  1$, $\norm{h}_{\infty} \leq  1$, where $\tilde{B}_m$'s are defined as in (\ref{def:Bm}) and (\ref{eq:defBessel}). 
\end{theorem}

This follows by minimal modification in the proof of Theorem \ref{prop:lip} once we prove the following two lemmas. Lemma \ref{lem:normPreMult} implies that the infinite norm of the inputs to each layer have the same bound. Lemma \ref{lem:nonexpansiveMult} gives a similar inequality to (\ref{eq:useBessel1.1}).
\begin{lemma}\label{lem:normPreMult}
(1) Let $g_{m,1}$ and $g_{m,2}$ be the two filters to be aggregated using multiplication with $\norm{g_{m,j}}_1 \leq 1$ for $j = 1,2$. We have the following: suppose $f_{m-1,1}$ and $f_{m-1,2}$ are the inputs to the filters respectively with $\norm{f_{m-1,j}}_{\infty} \leq 1$ for $j = 1,2$, then the output $f_m$ satisfies $\norm{f_m}_{\infty} \leq 1$; \\
(2) Let $g_m$ be a filter not to be aggregated with $\norm{g_m}_1 \leq 1$, then suppose $f_{m-1}$ is the input to the filter with $\norm{f_{m-1}}_{\infty} \leq 1$, we have the output $f_m$ satisfies $\norm{f_m}_{\infty} \leq 1$.
\end{lemma}
\begin{proof}
(2) directly follows from Young's Inequality. For (1), we have
\begin{equation*}
\begin{aligned}
\norm{f_m}_{\infty} ~=~ & \norm{\sigma_{m,1} ( f_{m-1,1} \ast g_{m,1} ) \cdot \sigma_{m,2} (f_{m-1,2} \ast g_{m,2})}_{\infty} \\
~\leq~ & \norm{f_{m-1,1} \ast g_{m,1}}_{\infty} \norm{f_{m-1,2} \ast g_{m,2}}_{\infty} \\
~\leq~ & \norm{f_{m-1,1}}_{\infty} \norm{f_{m-1,2}}_{\infty} \norm{g_{m,1}}_1 \norm{g_{m,2}}_1 \\
~\leq~ & 1 ~.
\end{aligned}
\end{equation*}
\end{proof}

\begin{lemma}\label{lem:nonexpansiveMult}
Let $g_{m,1}, g_{m,2}$ be the two filters to be aggregated using a  multiplication block with $\norm{g_{m,j}}_1 \leq 1$ for $j = 1,2$. We have the following: suppose $\{f_{m-1,j}\}_{j=1}^2$ and $\{h_{m-1,j}\}_{j=1}^2$ are two sets of inputs to those filters with infinite norm bounded by $1$, and $f_m$ and $h_m$ are the outputs respectively, then
\begin{equation*}
\norm{f_m - h_m}_2^2 \leq 2 \norm{ (f_{m-1,1} - h_{m-1,1}) \ast g_{m,1} }_2^2 + 2 \norm{ (f_{m-1,2} - h_{m-1,2}) \ast g_{m,2} }_2^2 ~.
\end{equation*}
\end{lemma}
\begin{proof}
\begin{equation*}
\begin{aligned}
& \norm{f_m - h_m}_2^2 \\ 
~=~ & \Vert \sigma_{m,1} ( f_{m-1,1} \ast g_{m,1} ) \sigma_{m,2} (f_{m-1,2} \ast g_{m,2}) - \\ 
& \qquad \sigma_{m,1} ( h_{m-1,1} \ast g_{m,1} ) \sigma_{m,2} (h_{m-1,2} \ast g_{m,2}) \Vert^2_2 \\
~=~ & \Vert \sigma_{m,1} ( f_{m-1,1} \ast g_{m,1} ) \sigma_{m,2} (f_{m-1,2} \ast g_{m,2}) - \\
& \qquad \sigma_{m,1} ( f_{m-1,1} \ast g_{m,1} ) \sigma_{m,2} (h_{m-1,2} \ast g_{m,2}) + \\
& \qquad \sigma_{m,1} ( f_{m-1,1} \ast g_{m,1} ) \sigma_{m,2} (h_{m-1,2} \ast g_{m,2}) - \\
& \qquad \sigma_{m,1} ( h_{m-1,1} \ast g_{m,1} ) \sigma_{m,2} ( h_{m-1,2} \ast g_{m,2}) \Vert_2^2 \\
~\leq~ & 2 \Vert \sigma_{m,1} ( f_{m-1,1} \ast g_{m,1} ) \sigma_{m,2} (f_{m-1,2} \ast g_{m,2}) - \\
& \qquad \sigma_{m,1} ( f_{m-1,1} \ast g_{m,1} ) \sigma_{m,2} (h_{m-1,2} \ast g_{m,2}) \Vert_2^2 + \\
& \qquad 2 \Vert \sigma_{m,1} ( f_{m-1,1} \ast g_{m,1} ) \sigma_{m,2} (h_{m-1,2} \ast g_{m,2}) - \\
& \qquad \sigma_{m,1} ( h_{m-1,1} \ast g_{m,1} ) \sigma_{m,2} (h_{m-1,2} \ast g_{m,2}) \Vert_2^2 \\
~\leq~ & 2 \Vert \sigma_{m,1} ( f_{m-1,1} \ast g_{m,1} ) \Vert_{\infty}^2 \Vert \sigma_{m,2} (f_{m-1,2} \ast g_{m,2}) - \\
& \qquad \sigma_{m,2} (h_{m-1,2} \ast g_{m,2}) \Vert_2^2 + \\
& \qquad 2 \norm{\sigma_{m,2} (h_{m-1,2} \ast g_{m,2} ) }_{\infty}^2 \Vert \sigma_{m,1} ( f_{m-1,1} \ast g_{m,1} ) - \\
& \qquad \sigma_{m,1} ( h_{m-1,1} \ast g_{m,1} ) \Vert_2^2 \\
~\leq~ & 2 \norm{f_{m-1,1}}_{\infty}^2 \norm{g_{m,1}}_1^2 \norm{ (f_{m-1,2} - h_{m-1,2}) \ast g_{m,2}}_2^2 + \\
& \qquad 2 \norm{h_{m-1,2}}_{\infty}^2 \norm{g_{m,2}}_1^2 \norm{ (f_{m-1,1} - h_{m-1,1}) \ast g_{m,1}}_2^2 \\
~\leq~ & 2 \norm{ (f_{m-1,1} - h_{m-1,1}) \ast g_{m,1} }_2^2 + \\
& \qquad 2 \norm{ (f_{m-1,2} - h_{m-1,2}) \ast g_{m,2} }_2^2 ~.
\end{aligned}
\end{equation*}
\end{proof}

For a general $f \in L^2(\mathbb{R}^d)$, as discussed in the end of Section 2, we can first let it go through a sigmoid-like function, then go through the scattering network. 

\subsection{Mixed aggregations}
The two types of aggregation blocks can be mixed together in the same networks (which is the common case in applications). The precise statement of the Lipschitz property becomes a little cumbersome to state in full generality. However, $L^2$-norm estimates can be combined using Theorem \ref{prop:lip}, \ref{prop:lipSquare}, \ref{prop:lipAggNorm} and \ref{prop:lipMult}. This is illustrated in the next section.

\section{Examples of estimating the Lipschitz constant}
We use three different approaches to estimate the Lipschitz constant. The first is by propagating backward from the outputs, regardless of what we have done above. The second is by directly applying what we have discussed above. The third is by deriving a lower bound, either because of the specifies of the network (the first example), or by numerical simulating (the second example). 

\subsection{A standard Scattering Network}
We first give an example of a standard scattering networks of three layers. The structure is as Figure 2.1 in \cite{Mallat12}. We consider the 1D case and the wavelet given by the Haar wavelets
\begin{equation*}
{\phi}(t) =
\begin{cases}
1, & \text{if}\ 0 \leq t < 1 \\
0, & \text{otherwise}
\end{cases}
\qquad \text{and} \qquad
{\psi}(t) =
\begin{cases}
1, & \text{if}\ 0 \leq t < 1/2 \\
-1, & \text{if}\ 1/2 \leq t < 1 \\
0, & \text{otherwise}
\end{cases}
\ .
\end{equation*}

In this section, 
the sinc function is defined as $\sinc(x) = \sin(\pi x) / (\pi x)$ if $x \neq 0$ and $0$ if $x = 0$.
 
We first look at real input functions. In this case the Haar wavelets $\phi$ and $\psi$ readily satisfies Equation (2.7) in \cite{Mallat12}. We take $J = 3$ in our example and consider all possible three-layer paths for $j = 0,-1,-2$. We have three branches from each node. Therefore we have outputs from $1+3+3^2+3^3 = 40$ nodes.


To convert the settings to our notations in this paper, we have a three-layer convolutional network (as in Section 2) for which the filters are given by $g_{1,l_1}, l_1 \in \{1,2,3\}$, $g_{2,l_2}, l_2 \in \{1,\cdots,9\}$ and $g_{3,l_3}, l_3 \in \{1,\cdots,27\}$, where
\begin{equation*}
g_{m,l} = 
\begin{cases}
\psi, & \text{if} \mod(l,3) = 1; \\
\psi_{2^{-1}}, & \text{if} \mod(l,3) = 2; \\
\psi_{2^{-2}}, & \text{if} \mod(l,3) = 0.
\end{cases}
\end{equation*}
$q = ((1,l_1),(2,l_2),(3,l_3))$ is a path if and only if $l_2 \in \{3l_1-k, k=1,2,3\}$ and $l_3 \in \{3l_2-k, k=1,2,3\}$. $q = ((1,l_1),(2,l_2))$ is a path if and only if $l_2 \in \{3l_1-k, k=1,2,3\}$. The set of all paths is
\begin{equation*}
\begin{aligned}
Q = ~& \{ \emptyset, \{(1,1)\}, \{(1,2)\}, \{(1,3)\}, \{(1,1),(2,1)\}, \{(1,1),(2,2)\}, \{(1,1),(2,3)\}, \\
& \{(1,2),(2,4)\}, \{(1,2),(2,5)\}, \{(1,2),(2,6)\}, \{(1,3),(2,7)\}, \{(1,3),(2,8)\}, \\
& \{(1,3),(2,9)\} ~\cup~ \{(1,l_1),(2,l_2),(3,l_3), 1 \leq l_1 \leq 3,  \\
& \qquad l_2 \in \{3l_1-k, k=1,2,3\}, l_3 \in \{3l_2-k, k=1,2,3\} \} ~.
\end{aligned}
\end{equation*}
Also, for the output generation, $\phi_1 = \phi_2 = \phi_3 = \phi_4 = 2^{-J}\phi(2^{-J}\cdot)$. An illustration of the network is as in Figure \ref{fig:struct6}.
\begin{figure}[ht!]
 \centering
 \includegraphics[width=0.8\textwidth]{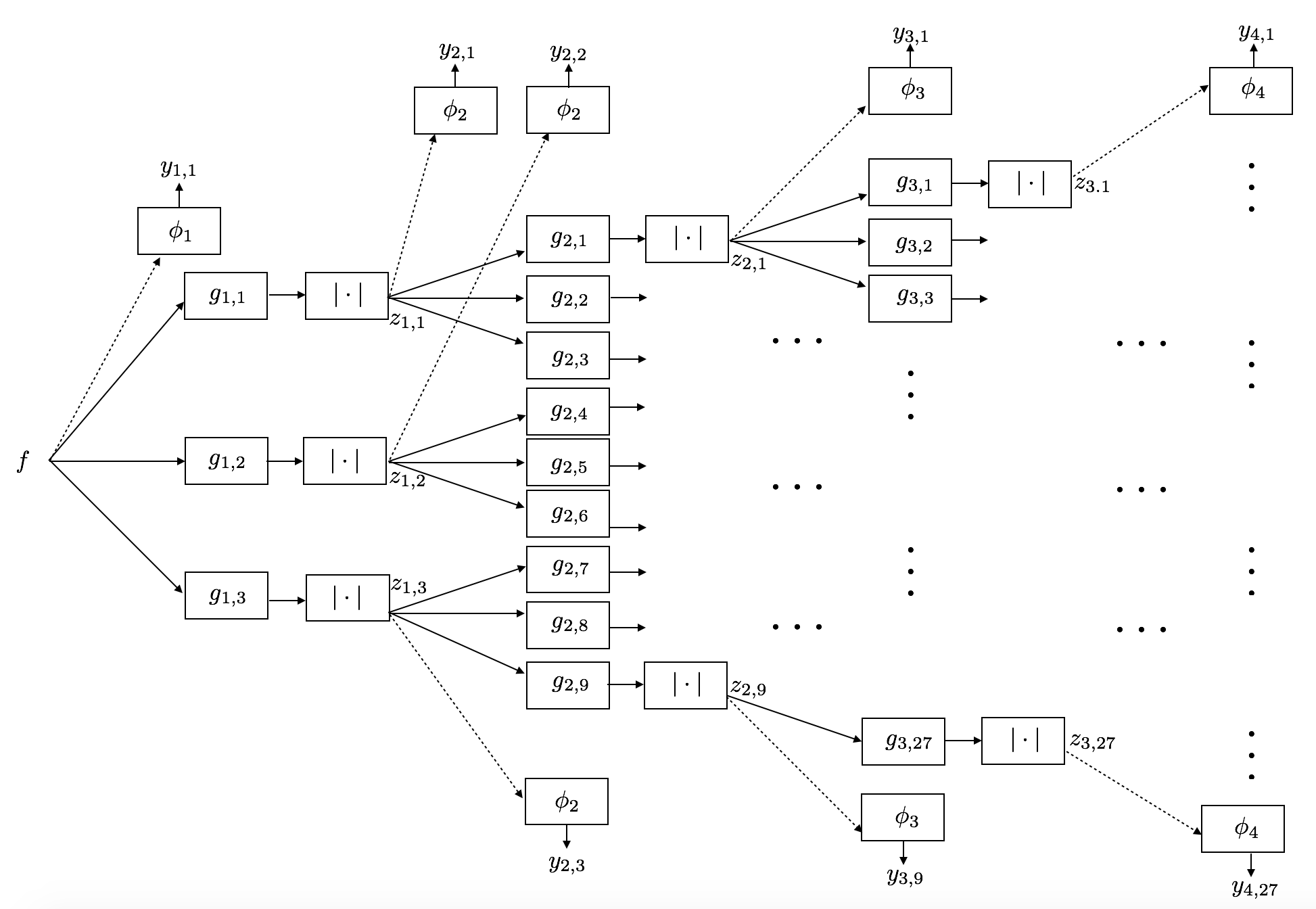}
 \caption{The scattering network in the example}
 \label{fig:struct6}
\end{figure}

The list of sets of filters $G_m^q$ and $G_m$ are
\begin{equation*}
\begin{aligned}
G_1^{\emptyset} = ~&~ \{g_{1,1},g_{1,2},g_{1,3}\} ~; \\
G_2^{(1,1)} = ~&~ \{g_{2,1},g_{2,2},g_{2,3}\} ~; \\
G_2^{(1,2)} = ~&~ \{g_{2,4},g_{2,5},g_{2,6}\} ~; \\
G_2^{(1,3)} = ~&~ \{g_{2,7},g_{2,8},g_{2,9}\} ~; \\
G_3^{((1,1),(2,1))} = ~&~ \{g_{3,1},g_{3,2},g_{3,3}\} ~; \\
& \cdots \\
G_3^{((1,3),(2,9))} = ~&~ \{g_{3,25},g_{3,26},g_{3,27}\} ~; \\
\end{aligned}
\end{equation*}
and
\begin{equation*}
\begin{aligned}
G_1 = ~&~ \{g_{1,1},g_{1,2},g_{1,3}\} ~; \\
G_2 = ~&~ \{g_{2,1}, \cdots, g_{2,9}\} ~; \\
G_3 = ~&~ \{g_{3,1}, \cdots, g_{3,27}\} ~.
\end{aligned}
\end{equation*}

\emph{The first approach.} We use backpropagation and the chain rule. Note that $\psi_{2^j}(t) = 2^j \psi(2^j t)$ and thus $\norm{\psi}_1 = \norm{\psi_{2^j}}_1 = 1$. Therefore $\norm{g_{m,l}}_1 = 1$ for all $m$, $l$. Similarly, $\norm{\phi_j}_1 = 1$ for all $j$. Let $y$'s denote the outputs and $z$'s denote the intermediate values, as marked in Figure \ref{fig:struct6}. Note that each $y$ is associated with a unique path. Consider two inputs $f$ and $\tilde{f}$, and $r \geq 1$. Take a path $q = ((1,l_1),(2,l_2),(3,l_3))$ we have
\begin{equation*}
\begin{aligned}
& \norm{y_{4,l_3}-\tilde{y}_{4,l_3}}_r = \norm{(z_{3,l_3}-\tilde{z}_{3,l_3}) \ast \phi_4}_r \leq \norm{z_{3,l_3}-\tilde{z}_{3,l_3}}_r \norm{\phi_4}_1 = \norm{z_{3,l_3}-\tilde{z}_{3,l_3}}_r ~; \\
& \norm{z_{3,l_3}-\tilde{z}_{3,l_3}}_r = \norm{\abs{z_{2,l_2} \ast g_{3,l_3}}-\abs{\tilde{z}_{2,l_2} \ast g_{3,l_3}}}_r \leq \\
& \qquad \qquad \qquad \norm{z_{2,l_2}-\tilde{z}_{2,l_2}}_r \norm{g_{3,l_3}}_1 = \norm{z_{2,l_2}-\tilde{z}_{2,l_2}}_r ~; \\
& \norm{z_{2,l_2}-\tilde{z}_{2,l_2}}_r = \norm{\abs{z_{1,l_1} \ast g_{2,l_2}}-\abs{\tilde{z}_{1,l_1} \ast g_{2,l_2}}}_r \leq \\
& \qquad \qquad \qquad \norm{z_{1,l_1}-\tilde{z}_{1,l_1}}_r \norm{g_{2,l_2}}_1 = \norm{z_{1,l_1}-\tilde{z}_{1,l_1}}_r ~; \\
& \norm{z_{1,l_1}-\tilde{z}_{1,l_3}}_r = \norm{\abs{f \ast g_{1,l_1}}-\abs{\tilde{f} \ast g_{1,l_1}}}_r \leq \norm{f-\tilde{f}}_r \norm{g_{1,l_1}}_1 = \norm{f-\tilde{f}}_r ~. \\
\end{aligned}
\end{equation*}
and similarly for all output $y_{m,l_m}$'s.
Therefore, we have
\begin{equation*}
||| \Phi(f) - \Phi(\tilde{f}) |||^2 = \sum_{m,l_m} \norm{y_{m,l_m}-\tilde{y}_{m,l_m}}_2^2 \leq 40 \norm{f-\tilde{f}}_2^2 ~.
\end{equation*}

\emph{The second approach.} According to the result from multi-resolution analysis, we have $\abs{\hat{\phi}_{2^{-J}}(\omega)} + \sum_{j=-2}^0 \abs{\hat{\psi}_{2^j}(\omega)}^2 \leq 1$ (plotted in Figure \ref{fig:result2}), we have $\tilde{B}_1 = \tilde{B}_2 = \tilde{B}_3 = \tilde{B}_4 = 1$. Indeed, we can compute that
\begin{equation*}
\begin{aligned}
\abs{\hat{\phi}_{2^{-J}}(\omega)} + \sum_{j=-2}^0 \abs{\hat{\psi}_{2^j}(\omega)}^2 = ~&~ \sinc^2(8\omega) + \sinc^2(\omega/2) \sin^2(\pi \omega/2) + \\
~&~ \sinc^2(\omega) \sin^2(\pi \omega) + \sinc^2(2\omega) \sin^2(2\pi \omega) ~.
\end{aligned}
\end{equation*}
Thus in this way, according to our discussion in Section 2, we have $|||\Phi(f)-\Phi(\tilde{f})|||^2 \leq \norm{f-\tilde{f}}_2^2$.

\begin{figure}[ht!]
 \centering
 \includegraphics[width=0.67\textwidth]{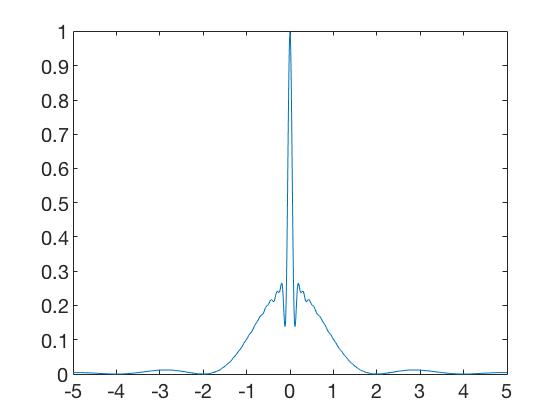}
 \caption{Plot of $\abs{\hat{\phi}_{2^{-J}}(\omega)} + \sum_{j=-2}^0 \abs{\hat{\psi}_{2^j}(\omega)}^2$}
 \label{fig:result2}
\end{figure}

\emph{The third approach.} A lower bound is derived by considering only the output $y_{1,1}$ from the input layer. Obviously
\begin{equation*}
||| \Phi(f) - \Phi(\tilde{f}) |||^2 \geq \norm{(f-\tilde{f}) \ast \phi_1}_1^2 ~.
\end{equation*}
Thus
\begin{equation*}
\sup_{f \neq \tilde{f}} \frac{|||\Phi(f)-\Phi(\tilde{f})|||^2}{\norm{f-\tilde{f}}_2^2} \geq \sup_{f \neq \tilde{f}} \frac{\norm{(f-\tilde{f}) \ast \phi_1}_1^2}{\norm{f-\tilde{f}}_2^2} = \norm{\hat{\phi}_1}_{\infty}^2 = 1 ~.
\end{equation*}
Therefore, $1$ is the exact Lipschitz bound (and Lipschitz constant) in our example.

\subsection{A general 3-layer network}
We now give an example of how to compute the Lipschitz constant as in Figure \ref{fig:struct5_1}. In Figure \ref{fig:struct5_1} $f$ is the input, $y$'s are the outputs and $z$'s are the intermediate values within the network. We assume that $p \geq 2$.

\begin{figure}[ht!]
 \centering
 \includegraphics[width=0.75\textwidth]{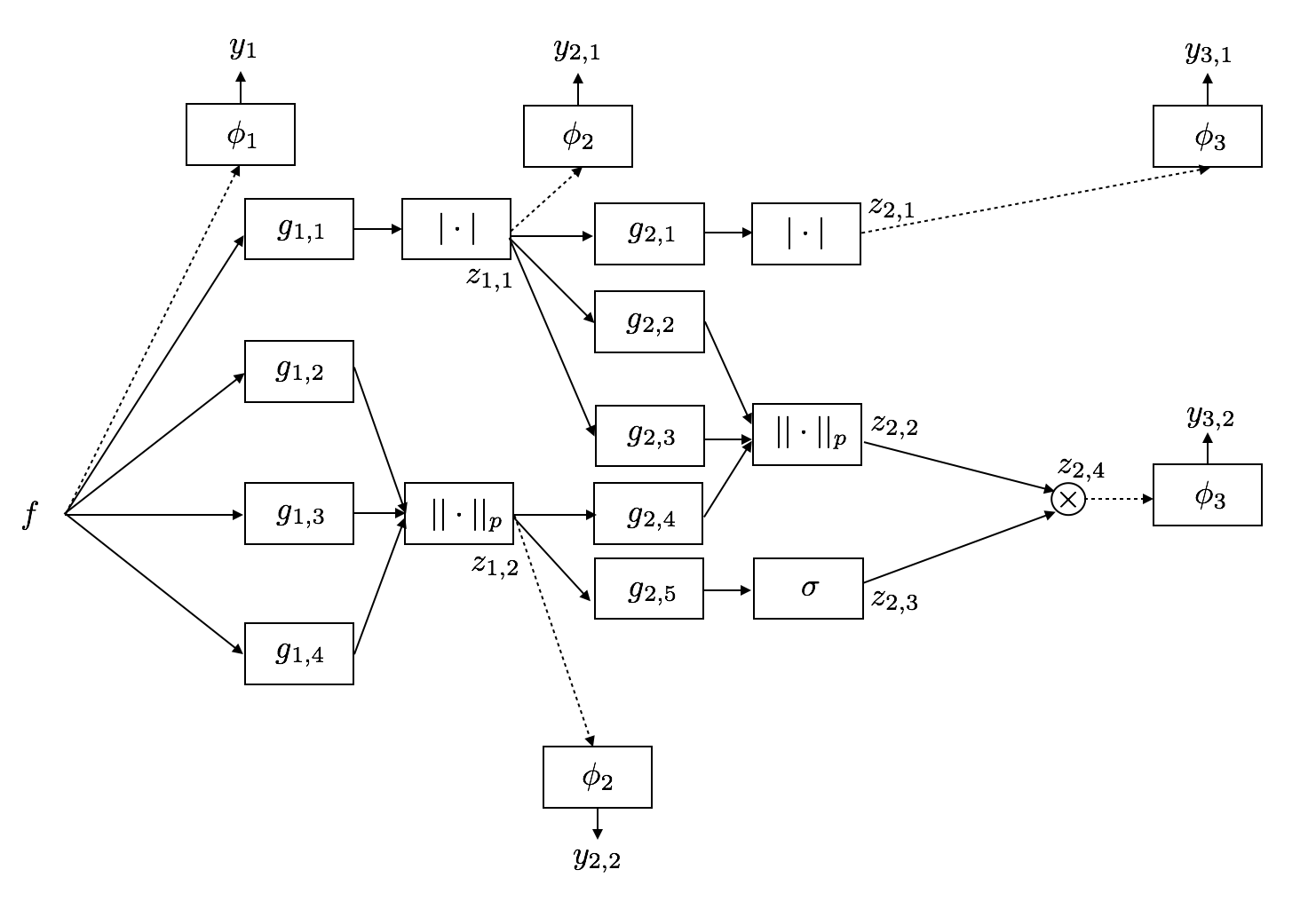}
 \caption{An example for computing the Lipschitz constant}
 \label{fig:struct5_1}
\end{figure}


Again we use three approaches to estimate the Lipschitz constant.

\emph{The first approach.}  In this approach we do not analyze the network by layers, but directly look at the outputs. We make use of the following rules: (1) backpropagation using the product rule and the chain rule; (2) each $p$-norm block is a multi-input-single-output nonlinear system with Lipschitz constant $1$ for each channel.

Take two signals $f$ and $\tilde{f}$. We use $\tilde{y}$'s and $\tilde{z}$'s to denote the outputs and intermediate values corresponding to $\tilde{f}$. Starting from the leftmost channels, we have for the first layer that
\begin{equation*}
\abs{y_1-\tilde{y}_1} = \abs{(f-\tilde{f}) \ast \phi_1} ~,
\end{equation*}
and thus for any $1 \leq r \leq \infty$,
\begin{equation}\label{eq:bp1}
\norm{y_1-\tilde{y}_1}_r \leq \norm{f-\tilde{f}}_r \norm{\phi_1}_1 ~.
\end{equation}

For the second layer we have
\begin{equation*}
\abs{y_{2,1}-\tilde{y}_{2,1}} = \abs{(z_{1,1}-\tilde{z}_{1,1}) \ast \phi_{2,2}} ~,
\end{equation*}
and thus
\begin{equation*}
\norm{y_{2,1}-\tilde{y}_{2,1}}_r \leq \norm{z_{1,1}-\tilde{z}_{1,1}}_r \norm{\phi_2}_1 ~.
\end{equation*}
With
\begin{equation*}
\norm{z_{1,1}-\tilde{z}_{1,1}}_r \leq \norm{f-\tilde{f}}_{r} \norm{g_{1,1}}_1 ~,
\end{equation*}
we have
\begin{equation}\label{eq:bp2_1}
\norm{y_{2,1}-\tilde{y}_{2,1}}_r \leq \norm{f-\tilde{f}}_r \norm{g_{1,1}}_1 \norm{\phi_2}_1 ~.
\end{equation}
Similarly,
\begin{equation*}
\norm{y_{2,2}-\tilde{y}_{2,2}}_r \leq \norm{z_{1,2}-\tilde{z}_{1,2}}_r \norm{\phi_2}_1 ~,
\end{equation*}
and with
\begin{equation*}
\begin{aligned}
\abs{z_{1,2}-\tilde{z}_{1,2}} ~=~ & \Big\vert \left(\abs{f \ast g_{1,2}}^p + \abs{f \ast g_{1,3}}^p + \abs{f \ast g_{1,4}}^p\right)^{1/p} - \\
& \qquad \left(\abs{\tilde{f} \ast g_{1,2}}^p + \abs{\tilde{f} \ast g_{1,3}}^p + \abs{\tilde{f} \ast g_{1,4}}^p\right)^{1/p} \Big\vert \\
~\leq~ & \left(\abs{(f-\tilde{f}) \ast g_{1,2}}^p + \abs{(f-\tilde{f}) \ast g_{1,3}}^p + \abs{(f-\tilde{f}) \ast g_{1,4}}^p\right)^{1/p} \\
~\leq~ & \abs{(f-\tilde{f}) \ast g_{1,2}} + \abs{(f-\tilde{f}) \ast g_{1,3}} + \abs{(f-\tilde{f}) \ast g_{1,4}}
\end{aligned}
\end{equation*}
we have
\begin{equation*}
\norm{z_{1,2}-\tilde{z}_{1,2}}_r \leq \norm{f-\tilde{f}} (\norm{g_{1,2}}_1 + \norm{g_{1,3}}_1 + \norm{g_{1,4}}_1) ~.
\end{equation*}
Therefore
\begin{equation}\label{eq:bp2_2}
\norm{y_{2,2}-\tilde{y}_{2,2}}_r \leq \norm{f-\tilde{f}} (\norm{g_{1,2}}_1 + \norm{g_{1,3}}_1 + \norm{g_{1,4}}_1) \norm{\phi_2}_1 ~.
\end{equation}

For the third layer we have
\begin{equation*}
\norm{y_{3,1}-\tilde{y}_{3,1}}_r \leq \norm{z_{2,1}-\tilde{z}_{2,1}}_r \norm{\phi_3}_1 ~.
\end{equation*}
With
\begin{equation*}
\norm{z_{2,1}-\tilde{z}_{2,1}}_r \leq \norm{z_{1,1}-\tilde{z}_{1,1}}_r \norm{g_{2,1}}_1 ~,
\end{equation*}
we have
\begin{equation}\label{eq:bp3_1}
\norm{y_{3,1}-\tilde{y}_{3,1}}_r \leq \norm{f-\tilde{f}}_r \norm{g_{1,1}}_1 \norm{g_{2,1}}_1 \norm{\phi_3}_1 ~.
\end{equation}
Also,
\begin{equation*}
\begin{aligned}
\abs{z_{2,2}-\tilde{z}_{2,2}} ~=~ & \Big\vert \left(\abs{z_{1,1} \ast g_{2,2}}^p + \abs{z_{1,1} \ast g_{2,3}}^p + \abs{z_{1,2} \ast g_{2,4}}^p\right)^{1/p} - \\
& \qquad \left(\abs{\tilde{z}_{1,1} \ast g_{2,2}}^p + \abs{\tilde{z}_{1,1} \ast g_{2,3}}^p + \abs{\tilde{z}_{1,2} \ast g_{2,4}}^p\right)^{1/p} \Big\vert \\
~\leq~ & ( \abs{(z_{1,1}-\tilde{z}_{1,1}) \ast g_{2,2}}^p + \abs{(z_{1,1}-\tilde{z}_{1,1}) \ast g_{2,3}}^p + \\
& \qquad \abs{(z_{1,2}-\tilde{z}_{1,2}) \ast g_{2,4}}^p )^{1/p} \\
~\leq~ & \abs{(z_{1,1}-\tilde{z}_{1,1}) \ast g_{2,2}} + \abs{(z_{1,1}-\tilde{z}_{1,1}) \ast g_{2,3}} + \abs{(z_{1,2}-\tilde{z}_{1,2}) \ast g_{2,4}} ~,
\end{aligned}
\end{equation*}
which gives
\begin{equation*}
\norm{z_{2,2}-\tilde{z}_{2,2}}_r \leq \norm{z_{1,1}-\tilde{z}_{1,1}}_r (\norm{g_{2,2}}_1+\norm{g_{2,3}}_1) + \norm{z_{1,2}-\tilde{z}_{1,2}}_r \norm{g_{2,4}}_1 ~.
\end{equation*}
A more obvious relation is 
\begin{equation*}
\norm{z_{2,3}-\tilde{z}_{2,3}}_r \leq \norm{z_{1,2}-\tilde{z}_{1,2}}_r \norm{g_{2,5}}_1 ~.
\end{equation*}
Under conditions in Theorem \ref{prop:lipMult}, we have
\begin{equation*}
\begin{aligned}
\norm{z_{2,4}-\tilde{z}_{2,4}}_r ~=~ & \norm{z_{2,3}z_{2,2}-\tilde{z}_{2,3}\tilde{z}_{2,2}}_r \\
~=~ & \norm{z_{2,3}z_{2,2}-\tilde{z}_{2,3}z_{2,2}+\tilde{z}_{2,3}z_{2,2}-\tilde{z}_{2,3}\tilde{z}_{2,2}}_r \\
~\leq~ & \norm{z_{2,3}-\tilde{z}_{2,3}}_r \norm{z_{2,2}}_{\infty} + \norm{\tilde{z}_{2,3}}_{\infty} \norm{z_{2,2}-\tilde{z}_{2,2}}_r \\
~\leq~ & \norm{z_{2,2}-\tilde{z}_{2,2}}_r+\norm{z_{2,3}-\tilde{z}_{2,3}}_r ~,
\end{aligned}
\end{equation*}
and consequently we have
\begin{equation}\label{eq:bp3_2}
\begin{aligned}
\norm{y_{3,2}-\tilde{y}_{3,2}}_r ~\leq~ & \norm{z_{2,4}-\tilde{z}_{2,4}}_r \norm{\phi_3}_1 \\
~\leq~ & (\norm{z_{2,2}-\tilde{z}_{2,2}}_r+\norm{z_{2,3}-\tilde{z}_{2,3}}_r) \norm{\phi_3}_1 \\
~\leq~ & \norm{z_{1,1}-\tilde{z}_{1,1}}_r (\norm{g_{2,2}}_1+\norm{g_{2,3}}_1)\norm{\phi_3}_1 + \\
& \qquad \norm{z_{1,2}-\tilde{z}_{1,2}}_r (\norm{g_{2,4}}_1+\norm{g_{2,5}}_1) \norm{\phi_3}_1 \\
~\leq~ & \norm{f-\tilde{f}}_r \Big( \norm{g_{1,1}}_1 (\norm{g_{2,2}}_1+\norm{g_{2,3}}_1)+ \\
& \qquad (\norm{g_{1,2}}_1+\norm{g_{1,3}}_1+\norm{g_{1,4}}_1)(\norm{g_{2,4}}_1+\norm{g_{2,5}}_1) \Big) \norm{\phi_3}_1 ~.
\end{aligned}
\end{equation}

Collecting (\ref{eq:bp1})-(\ref{eq:bp3_2}) we have
\begin{equation*}
\begin{aligned}
\sum_{m,l} \norm{y_{m,l}-\tilde{y}_{m,l}}_r ~\leq~ & \norm{f-\tilde{f}}_r \bigg( \norm{\phi_1}_1 + \norm{g_{1,1}}_1 \norm{\phi_2}_1 + \\
& \qquad (\norm{g_{1,2}}_1+\norm{g_{1,3}}_1+\norm{g_{1,4}}_1) \norm{\phi_2}_1 + \\
& \qquad \norm{g_{1,1}}_1 \norm{g_{2,1}}_1 \norm{\phi_3}_1 + \Big( \norm{g_{1,1}}_1 (\norm{g_{2,2}}_1+\norm{g_{2,3}}_1)+ \\
& \qquad (\norm{g_{1,2}}_1+\norm{g_{1,3}}_1+\norm{g_{1,4}}_1)(\norm{g_{2,4}}_1+\norm{g_{2,5}}_1) \Big) \norm{\phi_3}_1 \bigg) \\
~=~ & \norm{f-\tilde{f}}_r \bigg( \norm{\phi_1}_1 + (\norm{g_{1,1}}_1 +\norm{g_{1,2}}_1+\norm{g_{1,3}}_1+\norm{g_{1,4}}_1) \\
& \qquad \norm{\phi_2}_1 + \Big( \norm{g_{1,1}}_1 (\norm{g_{2,1}}_1+\norm{g_{2,2}}_1+\norm{g_{2,3}}_1)+ \\
& \qquad (\norm{g_{1,2}}_1+\norm{g_{1,3}}_1+\norm{g_{1,4}}_1)(\norm{g_{2,4}}_1+\norm{g_{2,5}}_1) \Big) \norm{\phi_3}_1 \bigg) ~.
\end{aligned}
\end{equation*}
On the other hand we also have
\begin{equation}\label{ex:bp}
\begin{aligned}
|||\Phi(f)-\Phi(\tilde{f})|||^2 ~=~ & \sum_{m,l} \norm{y_{m,l}-\tilde{y}_{m,l}}_2^2 \\
~\leq~ & \norm{f-\tilde{f}}_2^2 \bigg( \norm{\phi_1}_1^2 + \norm{g_{1,1}}_1^2 \norm{\phi_2}_1^2 + \\
& \qquad (\norm{g_{1,2}}_1+\norm{g_{1,3}}_1+\norm{g_{1,4}}_1)^2 \norm{\phi_2}_1^2 + \\
& \qquad \norm{g_{1,1}}_1^2 \norm{g_{2,1}}_1^2 \norm{\phi_3}_1^2 + \Big( \norm{g_{1,1}}_1 (\norm{g_{2,2}}_1+\norm{g_{2,3}}_1)+ \\
& \qquad (\norm{g_{1,2}}_1+\norm{g_{1,3}}_1+\norm{g_{1,4}}_1)(\norm{g_{2,4}}_1+\norm{g_{2,5}}_1) \Big)^2 \norm{\phi_3}_1^2 \bigg) ~.
\end{aligned}
\end{equation}

\emph{The second approach.} To apply our formula, we first add $\delta$'s and form a network as in Figure \ref{fig:struct5_2}. We have a three-layer network and as we have discussed, we can compute, since $p \geq 2$, that
\begin{equation*}
\begin{aligned}
\tilde{B}_1 &~= \norm{\abs{\hat{g}_{1,1}}^2+\abs{\hat{g}_{1,2}}^2+\abs{\hat{g}_{1,3}}^2+\abs{\hat{g}_{1,4}}^2+\abs{\hat{\phi}_1}^2}_{\infty}; \\
\tilde{B}_2 &~= \max \left\{ 1, \norm{\abs{\hat{g}_{2,1}}^2+\abs{\hat{g}_{2,2}}^2+\abs{\hat{g}_{2,3}}^2+\abs{\hat{\phi}_2}^2}_{\infty} , \norm{\abs{\hat{g}_{2,4}}^2+\abs{\hat{g}_{2,5}}^2+\abs{\hat{\phi}_2}^2}_{\infty} \right\}; \\
\tilde{B}_3 &~= \max \left\{ 2 , \norm{\hat{\phi}_3}_{\infty}^2 \right\}; \\
\tilde{B}_4 &~= \max \left\{ 1, \norm{\hat{\phi}_3}_{\infty}^2 \right\}.
\end{aligned}
\end{equation*}
Then the Lipschitz constant is given by $(\tilde{B}_1 \tilde{B}_2 \tilde{B}_3 \tilde{B}_4)^{1/2}$, that is, 
\begin{equation}\label{ex:formula}
|||\Phi(f)-\Phi(\tilde{f})|||^2 \leq (\tilde{B}_1 \tilde{B}_2 \tilde{B}_3 \tilde{B}_4) \norm{f-\tilde{f}}_2^2 ~.
\end{equation}

\begin{figure}[ht!]
 \centering
 \includegraphics[width=0.75\textwidth]{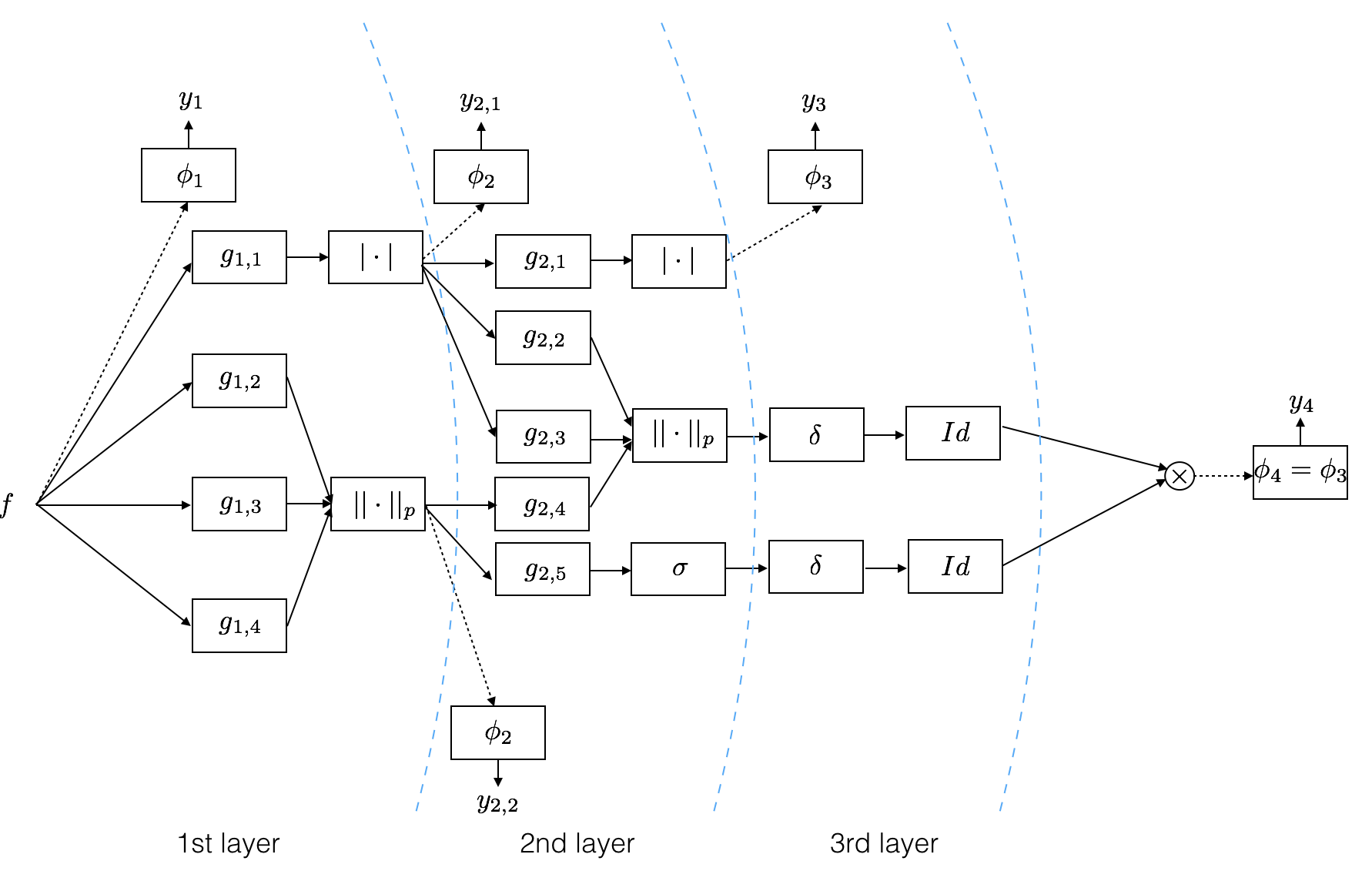}
 \caption{An equivalent reformulation of the same network as in Figure \ref{fig:struct5_1}}
 \label{fig:struct5_2}
\end{figure}

\emph{The third approach.} In general (\ref{ex:formula}) provides a more optimal bound than (\ref{ex:bp}) because the latter does not consider the intrinsic relations of the filters that are grouped together in the same layer. The actual Lipschitz bound can depend on the actual design of filters, not only on the Bessel bounds. We do a numerical experiment in which the Fourier transform of the filters in the same layer are the (smoothed) characteristic functions supported disjointly in the frequency domain. 

Define $F(\omega) = \exp(4\omega^2/(4\omega^2-1)) \cdot \chi_{(-1/2,0)}(\omega)$, and $G(\omega) = F(-\omega)$. The fourier transform of the filters are defined to be
\begin{equation*}
\begin{aligned}
\hat{\phi}_1(\omega) ~=~ & F(\omega+1)+\chi_{(-1,1)}(\omega)+G(\omega-1) \\
\hat{g}_{1,1}(\omega) ~=~ & F(\omega+3)+\chi_{(-3,-2)}(\omega)+G(\omega+2)+F(\omega-2)+\chi_{(2,3)}(\omega)+G(\omega-3) \\
\hat{g}_{1,2}(\omega) ~=~ & F(\omega+5)+\chi_{(-5,-4)}(\omega)+G(\omega+4)+F(\omega-4)+\chi_{(4,5)}(\omega)+G(\omega-5) \\
\hat{g}_{1,3}(\omega) ~=~ & F(\omega+7)+\chi_{(-7,-6)}(\omega)+G(\omega+6)+F(\omega-6)+\chi_{(6,7)}(\omega)+G(\omega-7) \\
\hat{g}_{1,4}(\omega) ~=~ & F(\omega+9)+\chi_{(-9,-8)}(\omega)+G(\omega+8)+F(\omega-8)+\chi_{(8,9)}(\omega)+G(\omega-9) \\
\hat{\phi}_2(\omega) ~=~ & F(\omega+2)+\chi_{(-2,2)}(\omega)+G(\omega-2) \\
\hat{g}_{2,1}(\omega) ~=~ & F(\omega+4)+\chi_{(-4,-3)}(\omega)+G(\omega+3)+F(\omega-3)+\chi_{(3,4)}(\omega)+G(\omega-4) \\
\hat{g}_{2,2}(\omega) ~=~ & F(\omega+6)+\chi_{(-6,-5)}(\omega)+G(\omega+5)+F(\omega-5)+\chi_{(5,6)}(\omega)+G(\omega-6) \\
\hat{g}_{2,3}(\omega) ~=~ & F(\omega+8)+\chi_{(-8,-7)}(\omega)+G(\omega+7)+F(\omega-7)+\chi_{(7,8)}(\omega)+G(\omega-8) \\
\hat{g}_{2,4}(\omega) ~=~ & F(\omega+5)+\chi_{(-5,-3)}(\omega)+G(\omega+3)+F(\omega-3)+\chi_{(3,5)}(\omega)+G(\omega-5) \\
\hat{g}_{2,5}(\omega) ~=~ & F(\omega+8)+\chi_{(-8,-6)}(\omega)+G(\omega+6)+F(\omega-6)+\chi_{(6,8)}(\omega)+G(\omega-8) \\
\hat{\phi}_3(\omega) ~=~ & F(\omega+9)+\chi_{(-9,9)}(\omega)+G(\omega-9) \\
\end{aligned}
\end{equation*}
Then each function is in $C_{C}^{\infty}(\hat{\mathbb{R}})$.

We numerically compute the $L^1$ norms of the inverse transform of  the above functions using IFFT and numerical integration with stepsize $0.025$: $\norm{\phi_1}_1 = 1.8265$, $\norm{g_{1,1}}_1 = 2.0781$, $\norm{g_{1,2}}_1 = 2.0808$, $\norm{g_{1,3}}_1 = 2.0518$, $\norm{g_{1,4}}_1 = 2.0720$, $\norm{\phi_2}_1 = 2.0572$, $\norm{g_{2,1}}_1 = 2.0784$, $\norm{g_{2,2}}_1 = 2.0734$, $\norm{g_{2,3}}_1 = 2.0889$, $\norm{g_{2,4}}_1 = 2.2390$, $\norm{g_{2,5}}_1 = 2.3175$, $\norm{\phi_3}_1 = 2.6378$. Then the constant on the right-hand side of Inequality (\ref{ex:bp}) is $966.26$, and by taking the square root we get the Lipschitz bound computed using the first approach is equal to $\Gamma_1 = 98.3$.

It is no effort to conclude that in the second approach, $\tilde{B}_1 = \tilde{B}_2 = \tilde{B}_4 = 1$ and $\tilde{B}_3 = 2$. Therefore the Lipschitz bound computed using the second approach is $\Gamma_2 = \sqrt{2}$. Note that in this example the conditions in Lemma \ref{lem:nonexpansiveMult} is satisfied.

The experiment suggests that the Lipschitz bound associated with our setting of filters is $\Gamma_3 = 1.1937$. We numerically compute the output of the network and record the largest ratio $|||\Phi(f)-\Phi(\tilde{f})|||/||f-\tilde{f}||_2$ over one million iterations. Numerically, we consider the range $[-20,20]$ for both the time domain and the frequency domain and take stepsize to be $0.025$. For each iteration we generate two randomly signals on $[-20,20]$ with stepsize $1$ and then upsample to the same scale with stepsize $0.025$.

We conclude that the na\"ive first approach may lead to a much larger Lipschitz bound for analysis, and the second approach gives a more reasonable estimation.

\section*{Acknowledgments}
The first author was partially supported by NSF Grant DMS-1413249 and ARO Grant W911NF-16-1-0008. The third author was partially supported by NSF Grant DMS-1413249.

\nocite{*}
\bibliography{ams2016ref}
\bibliographystyle{amsplain}

\end{document}